\documentclass[twoside,11pt]{article}
\usepackage{pgm} 

\usepackage{float,MnSymbol}
\usepackage{tikz}
\usetikzlibrary{arrows}
\usetikzlibrary{decorations.markings}

\def\ci{\!\perp\!}
\def\nci{\!\not\perp\!}
\def\ra{\rightarrow}
\def\la{\leftarrow}
\def\aa{\leftrightarrow}

\def\ao{\leftarrow\!\!\!\!\!\multimap}

\def\oa{\mathrel{\reflectbox{\ensuremath{\ao}}}}

\newcommand{\comments}[1]{}
\newcommand{\wh}[1]{\widehat{#1}}

\tikzset{tt/.style={decoration={
  markings,
  mark=at position .485 with {\arrow{>}},
  mark=at position .515 with {\arrow{<}}},postaction={decorate}}}

\usepackage{appendix}

\usepackage{float}

\usepackage{graphicx}
\usepackage{subcaption}

\usepackage{threeparttable}
\usepackage{booktabs}
\usepackage{array}

\usepackage{amsmath}
\usepackage{amsfonts}
\usepackage{amssymb}
\usepackage{mathtools}

\usepackage{tikz}
\usetikzlibrary{shapes,arrows}
\usetikzlibrary{shapes,snakes}
\usetikzlibrary{fit,positioning}
\usetikzlibrary{shapes.geometric}
\usetikzlibrary{calc}
\usetikzlibrary{decorations.pathreplacing}

\tikzstyle{var} = [node distance = 5mm]
\tikzstyle{unobserved}=[var, dashed, circle, draw=black!80]
\tikzstyle{intervention}=[circle, double, draw =black!80, node distance = 5mm, inner sep=0.5mm]
\tikzstyle{chance}=[circle, minimum size = 10mm, thick, draw =black!80, node distance = 6mm, align=center, fill=white]
\tikzstyle{utility}=[diamond, aspect=1, minimum size = 10mm, thick, draw =black!80, node distance = 6mm,align=center, fill=white]
\tikzstyle{decision}=[rectangle, aspect=1, minimum size = 10mm, thick, draw =black!80, node distance = 6mm,align=center, fill=white]
\tikzstyle{gate}=[node distance = 10mm, align=center]
\tikzstyle{box}=[rectangle, inner sep=2.0mm,draw=black!100, ]
\tikzstyle{connect}=[-latex]
\tikzstyle{dconnect}=[latex-latex]
\tikzstyle{textnode} = []

\ShortHeadings{Effect Identification in ADMGs and GMs}{Pe\~{n}a and Bendtsen} 

\begin{document}

\title{Causal Effect Identification in Acyclic Directed Mixed Graphs and Gated Models}
\author{\name Jose M. Pe\~{n}a \email jose.m.pena@liu.se\\
\name Marcus Bendtsen \email marcus.bendtsen@liu.se\\
\addr Department of Computer and Information Science\\
Link{\"o}ping University (Sweden)}
\maketitle


\begin{abstract}
We introduce a new family of graphical models that consists of graphs with possibly directed, undirected and bidirected edges but without directed cycles. We show that these models are suitable for representing causal models with additive error terms. We provide a set of sufficient graphical criteria for the identification of arbitrary causal effects when the new models contain directed and undirected edges but no bidirected edge. We also provide a necessary and sufficient graphical criterion for the identification of the causal effect of a single variable on the rest of the variables. Moreover, we develop an exact algorithm for learning the new models from observational and interventional data via answer set programming. Finally, we introduce gated models for causal effect identification, a new family of graphical models that exploits context specific independences to identify additional causal effects.
\end{abstract}


\section{Introduction}\label{sec:introduction}

Undirected graphs (UGs), bidirected graphs (BGs), and directed and acyclic graphs (DAGs) have extensively been studied as representations of independence models. DAGs have also been studied as representation of causal models, because they can model asymmetric relationships between random variables. DAGs and UGs (respectively BGs) have been extended into chain graphs (CGs), which are graphs with possibly directed and undirected (respectively bidirected) edges but without semidirected cycles. Therefore, CGs can model both symmetric and asymmetric relationships between random variables. CGs with possibly directed and undirected edges may represent a different independence model depending on whether the Lauritzen-Wermuth-Frydenberg (LWF) or the Andersson-Madigan-Perlman (AMP) interpretation is considered \citep{Lauritzen1996,Anderssonetal.2001}. CGs with possibly directed and bidirected edges have a unique interpretation, the so-called multivariate regression (MVR) interpretation \citep{CoxandWermuth1996}. MVR CGs have been extended by (i) relaxing the semidirected acyclity constraint so that only directed cycles are forbidden, and (ii) allowing up to two edges between any pair of nodes. The resulting models are called original acyclic directed mixed graphs (oADMGs) \citep{Richardson2003}. AMP CGs have also been extended similarly \citep{Penna2016}. The resulting models are called alternative acyclic directed mixed graphs (aADMGs).

In this paper, we combine oADMGs and aADMGs into what we simply call ADMGs. These are graphs with possibly directed, undirected and bidirected edges but without directed cycles. Moreover, there can be up to three edges between any pair of nodes. This work complements the existing works for the following reasons. To our knowledge, the only mixed graphical models in the literature that subsume AMP CGs are the already mentioned aADMGs and the so-called marginal AMP CGs \citep{Penna2014}. However, marginal AMP CGs are simple graphs with possibly directed, undirected and bidirected edges but without semidirected cycles and, moreover, some constellations of edges are forbidden. Therefore, marginal AMP CGs do not subsume ADMGs. Likewise, no other family of mixed graphical models that we know of (e.g. oADMGs, summary graphs \citep{CoxandWermuth1996}, ancestral graphs \citep{RichardsonandSpirtes2002}, MC graphs \citep{Koster2002} or loopless mixed graphs \citep{SadeghiandLauritzen2014}) subsume AMP CGs and hence ADMGs. To see it, we refer the reader to the works by \citet[p. 1025]{RichardsonandSpirtes2002} and \citet[Section 4.1]{SadeghiandLauritzen2014}.

\begin{figure}
\begin{center}
\begin{tabular}{|c|c|c|}
\hline
DAG&oADMG&aADMG\\
\hline
\begin{tikzpicture}[inner sep=1mm]
\node at (0,0) (A) {$A$};
\node at (2,0) (B) {$B$};
\node at (1,1) (C) {$C$};
\node at (0,1) (UA) {$U_A$};
\node at (2,1) (UB) {$U_B$};
\node at (1,2) (UC) {$U_C$};
\path[->] (A) edge (B);
\path[->] (UA) edge (A);
\path[->] (UB) edge (B);
\path[->] (UC) edge (C);
\path[->] (UA) edge (UC);
\path[->] (UC) edge (UB);
\end{tikzpicture}
&
\begin{tikzpicture}[inner sep=1mm]
\node at (0,0) (A) {$A$};
\node at (1.5,0) (B) {$B$};
\node at (0.75,1) (C) {$C$};
\path[->] (A) edge (B);
\path[<->] (A) edge [bend left] (B);
\path[<->] (A) edge (C);
\path[<->] (B) edge (C);
\end{tikzpicture}
&
\begin{tikzpicture}[inner sep=1mm]
\node at (0,0) (A) {$A$};
\node at (1.5,0) (B) {$B$};
\node at (0.75,1) (C) {$C$};
\path[->] (A) edge (B);
\path[-] (A) edge (C);
\path[-] (B) edge (C);
\end{tikzpicture}\\
\hline
\end{tabular}
\end{center}\caption{Example where $p(B | \wh{A})$ is identifiable from the aADMG but not from the oADMG.}\label{fig:example1}
\end{figure}
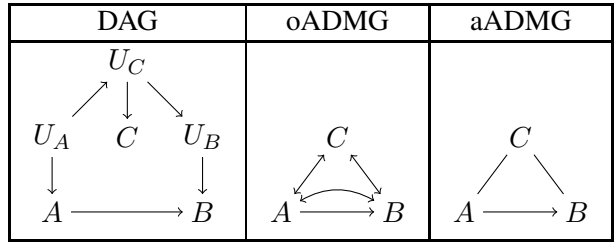

\begin{figure}
\begin{center}
\begin{tabular}{|c|c|c|}
\hline
DAG&oADMG&aADMG\\
\hline
\begin{tikzpicture}[inner sep=1mm]
\node at (0,0) (A) {$A$};
\node at (2,0) (B) {$B$};
\node at (1,1) (C) {$C$};
\node at (0,1) (UA) {$U_A$};
\node at (2,1) (UB) {$U_B$};
\node at (1,2) (UC) {$U_C$};
\path[->] (A) edge (B);
\path[->] (UA) edge (A);
\path[->] (UB) edge (B);
\path[->] (UC) edge (C);
\path[->] (UA) edge (UC);
\path[<-] (UC) edge (UB);
\end{tikzpicture}
&
\begin{tikzpicture}[inner sep=1mm]
\node at (0,0) (A) {$A$};
\node at (1.5,0) (B) {$B$};
\node at (0.75,1) (C) {$C$};
\path[->] (A) edge (B);
\path[<->] (A) edge (C);
\path[<->] (B) edge (C);
\end{tikzpicture}
&
\begin{tikzpicture}[inner sep=1mm]
\node at (0,0) (A) {$A$};
\node at (1.5,0) (B) {$B$};
\node at (0.75,1) (C) {$C$};
\path[->] (A) edge (B);
\path[-] (A) edge [bend left] (B);
\path[-] (A) edge (C);
\path[-] (B) edge (C);
\end{tikzpicture}\\
\hline
\end{tabular}
\end{center}\caption{Example where $p(B | \wh{A})$ is identifiable from the oADMG but not from the aADMG.}\label{fig:example1b}
\end{figure}
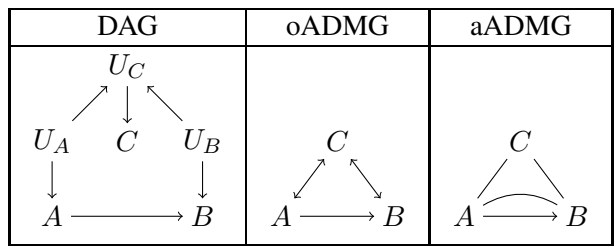

In addition to represent independence models, some of the families of graphical models mentioned above have been used for causal effect identification, i.e. to determine if the causal effect of an intervention is computable from observed quantities. For instance, Pearl's approach to causal effect identification makes use of oADMGs to represent causal models over the observed variables \citep{Pearl2009}. The directed edges represent potential causal relationships, whereas the bidirected edges represent potential confounding, i.e. a latent common cause. A key feature of Pearl's approach is that no assumption is made about the functional form of the causal relationships. That is, each variable $A$ is an unconstrained function of its observed causes $Pa(A)$ and its unobserved causes $U_A$, i.e. $A = f(Pa(A), U_A)$. In this paper, we study causal effect identification under the assumption that $A = f(Pa(A)) + U_A$, i.e. under the assumption of additive errors. This is a rather common assumption in causal discovery, e.g. see \citep{BuhPetErn14,Janzingetal.2009,Petersetal.2014}. Specifically, we show that ADMGs are suitable for representing such causal models: An undirected edge between two nodes represents potential dependence between their error terms given the rest of the error terms, as opposed to a bidirected edge that represents potential marginal dependence due to confounding. The reason for studying ADMGs for causal effect identification is that we may identify more causal effects from them than from oADMGs, since the former are tailored to the additive error assumption. We illustrate this question with the example in Figure \ref{fig:example1}, which is borrowed from \citet{Penna2016}. The ADMGs in the figure represent the causal model over the observed variables represented by the DAG. The oADMG is derived from the DAG by keeping the directed edges between observed variables, and adding a bidirected edge between two observed variables if and only if they have a confounder \citep[Section 5]{TianandPearl2002b}. The aADMG is derived from the DAG by keeping the directed edges between observed variables, and adding an undirected edge between two observed variables if and only if their unobserved causes are not separated in the DAG given the unobserved causes of the rest of the observed variables. Clearly, the effect on $B$ of an intervention on $A$, i.e. $p(B | \wh{A})$, is not identifiable from the oADMG \citep[p. 94]{Pearl2009}, but it is identifiable from the aADMG and is given by
\[
p(B | \wh{A}) = \sum_c p(B | A, c) p(c).
\]
To see it, recall that we assume additive noise. This implies that $C$ determines $U_C$, which blocks the path $A \la U_A \ra U_C \ra U_B \ra B$ in the DAG. This can also be seen directly in the aADMG, as $C$ blocks the path $A - C - B$. Therefore, we can identify the desired causal effect by just adjusting for $C$, since $C$ blocks all non-causal paths from $A$ to $B$. It is worth mentioning that there are also cases where the oADMG allows for causal effect identification whereas the aADMG does not. One such case is shown in Figure \ref{fig:example1b}, where we have just replaced the edge $U_C \ra U_B$ in Figure \ref{fig:example1} with the edge $U_C \la U_B$. Specifically, $p(B | \wh{A})$ is not identifiable from the aADMG by Theorem \ref{the:children2} in this article, whereas it is identifiable from the oADMG, i.e. $p(B | \wh{A}) = p(B | A)$. Therefore, oADMGs and aADMGs are more complementary than competing causal models. To further illustrate our point, we make the example in Figure \ref{fig:example1} more concrete by turning it into the following invented gambling game:
\begin{align} \nonumber \label{eq:game}
U_A &\sim N(0,\sigma)\\ \nonumber
U_C &\sim N(U_A,\sigma)\\ \nonumber
U_B &\sim N(U_C,\sigma)\\ \nonumber
A &= U_A\\ \nonumber
C &= U_C\\
B &= A + U_B.
\end{align}
In other words, $U_A$, $U_B$ and $U_C$ represent three unobserved randomly chosen numbers that determine three observed numbers represented by $A$, $B$ and $C$. Now, we are told that $A$ has been set to value $a$ independently of the value of $U_A$. We are asked to bet on the value of $B$. To make an informed bet, we would like to know $p(B | \wh{A}=a)$. However, all we are provided with is $p(A,B,C)$ and the oADMG or the aADMG in Figure \ref{fig:example1}. As discussed above, $p(A,B,C)$ and the aADMG are enough for our purpose, whereas $p(A,B,C)$ and the oADMG are not. As also discussed above, this situation can be reversed if we modify the game as follows:
\begin{align} \nonumber \label{eq:game2}
U_C &\sim N(U_A + U_B,\sigma)\\
U_B &\sim N(0,\sigma)
\end{align}
that is, $p(A,B,C)$ and the oADMG are now enough for our purpose, whereas $p(A,B,C)$ and the aADMG are not.

As mentioned, aADGMs were proposed by \citet{Penna2016}, who mainly studied them as representation of statistical independence models. In particular, their global, local and pairwise Markov properties were studied. Their usage to represent causal models was also discussed, but no formal criteria or calculus for causal effect identification from them were given. This paper is a first step to fill that gap. In particular, we present a calculus similar to the {\it do}-calculus by \citet{Pearl2009}, and a decomposition of the distribution over the observed random variables similar to the Q-decomposition by \citet{TianandPearl2002a,TianandPearl2002b}. From this calculus and decomposition, we derive a set of sufficient graphical criteria for the identification of arbitrary causal effects from aADMGs. We also provide a necessary and sufficient graphical criterion for the identification of the causal effect of a single variable on the rest of the variables. Our ambition is to extend these results to ADMGs in the future.

\begin{figure}
\begin{center}
\begin{tikzpicture}[inner sep=1mm]
\begin{scope}
\node at (0,0) (A) {$A$};
\node at (1.5,0) (B) {$B$};
\node at (0.75,1) (C) {$C$};
\path[->] (A) edge (B);
\path[-] (A) edge (C);
\path[-] (B) edge (C);
\node[box, inner sep=8pt, fit= (A) (B) (C), label=left:$$] (O1) {};
\end{scope}
\begin{scope}[xshift=5.0cm]
\node at (0,0) (A) {$A$};
\node at (1.5,0) (B) {$B$};
\node at (0.75,1) (C) {$C$};
\path[->] (A) edge (B);
\path[<->] (A) edge (C);
\path[<->] (B) edge (C);
\node[box,  inner sep=8pt, fit= (A) (B) (C), label=left:$$] (O2) {};
\end{scope}
\node[gate,right = of O1,xshift=-0.5cm,yshift=0.7cm] (G1) {$\wh{A}  \leq 0$};
\path (O1) edge [connect] (G1);
\path (G1) edge [connect] (O2);
\node[gate,left=of O2,xshift=0.5cm,yshift=-0.7cm] (G2) {$\wh{A} > 0$};
\path (O2) edge [connect] (G2);
\path (G2) edge [connect] (O1);
\end{tikzpicture}
\end{center}
\caption{Gated model with two contexts. The aADMG is used in the context $\wh{A}>0$, whereas the oADMG is used in the context $\wh{A} \leq 0$. This ensures that $p(B | \wh{A})$ is identifiable.}\label{fig:gm}
\end{figure}
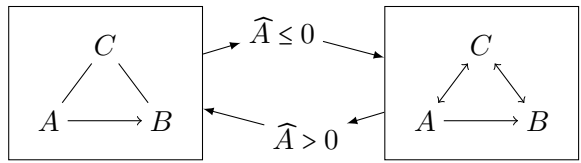

We have seen above that choosing an appropriate causal model is crucial for identifying the causal effect of interest. However, the domain being modeled may consist of different regimes or contexts, and the appropriate causal model for one regime may not be quite so for another. For instance, consider again the gambling games described above. Assume that we are told that the original gambling game (Equation \ref{eq:game}) applies when $\wh{A} > 0$, whereas the modified one (Equation \ref{eq:game2}) applies when $\wh{A} \leq 0$. Therefore, the aADMG in Figure \ref{fig:example1} should be preferred in the context $\wh{A}>0$ since it allows identifying $p(B | \wh{A})$. For the same reason, the oADMG in Figure \ref{fig:example1b} should be preferred in the context $\wh{A} \leq 0$. One may hope that combining aADMGs and oADMGs into ADMGs will help to identify a single model that is appropriate for both contexts. Whether this will be the case is not clear to us as of today (as mentioned above, we do not know yet how to perform causal effect identification in ADMGs). In the meantime, we propose to solve this problem with the help of the gated model in Figure \ref{fig:gm}. In this model, we have the aADMG and the oADMG connected with gates such that when $\wh{A}>0$ we shall use the aADMG to the left, and when $\wh{A} \leq 0$ then we shall use the oADMG to the right. Gated models have previously been used to model independence models induced by dynamical systems with recurrent regimes \citep{bendtsen-ijar,bendtsen-online,bendtsen-regime-baseball}. However, this is the first work to use them for causal effect identification. We will moreover show how gated models allow us to represent causal phenomena such as unstable interventions, non-deterministic outcomes of interventions, and mechanism dependent outcome of interventions.

The rest of the paper is organized as follows. Section \ref{sec:preliminaries} introduces the notation and some preliminary concepts. Section \ref{sec:separation} introduces two equivalent separation criteria for ADMGs, which define their semantics as a formalism to represent independence models. Section \ref{sec:interpretationandlearning} provides an intuitive causal interpretation of ADMGs as systems of structural equations with additive and correlated errors. Section \ref{sec:interpretationandlearning} also describes an exact algorithm for learning ADMGs from observational and interventional data via answer set programming \citep{gelfond_1988,DBLP:journals/amai/Niemela99,DBLP:journals/ai/SimonsNS02}. Section \ref{sec:identification} presents graphical criteria for causal effect identification in aADMGs. Section \ref{sec:csi} is devoted to gated models for causal effect identification. Finally, Section \ref{sec:conclusions} closes the paper with a summary and future lines of research.

\begin{figure}
\centering
\begin{tabular}{|c|c|c|}
\hline
\begin{tikzpicture}[inner sep=1mm]
\node at (0,0) (A) {$A$};
\node at (1,0) (B) {$B$};
\node at (2,0) (D) {$E$};
\node at (2,-1) (F) {};
\path[->] (A) edge (B);
\path[->] (B) edge (D);
\path[-] (B) edge [bend left] (D);
\path[<->] (B) edge [bend right] (D);
\end{tikzpicture}
&
\begin{tikzpicture}[inner sep=1mm]
\node at (0,0) (A) {$A$};
\node at (1,0) (B) {$B$};
\node at (2,0) (C) {$C$};
\node at (3,0) (D) {$E$};
\node at (2,-1) (F) {};
\path[->] (A) edge (B);
\path[->] (B) edge (C);
\path[->] (C) edge (D);
\path[-] (A) edge [bend left] (C);
\path[<->] (B) edge [bend left] (D);
\end{tikzpicture}
&
\begin{tikzpicture}[inner sep=1mm]
\node at (0,0) (A) {$A$};
\node at (1,0) (B) {$B$};
\node at (2,0) (C) {$C$};
\node at (3,0) (D) {$D$};
\node at (4,0) (E) {$E$};
\node at (2,-1) (F) {$F$};
\path[-] (A) edge (B);
\path[-] (B) edge (C);
\path[-] (C) edge (D);
\path[-] (D) edge (E);
\path[<->] (A) edge [bend left] (D);
\path[<->] (B) edge [bend left] (E);
\path[<->] (C) edge (F);
\end{tikzpicture}
\\
\hline
\end{tabular}\caption{Examples of ADMGs.}\label{fig:example}
\end{figure}
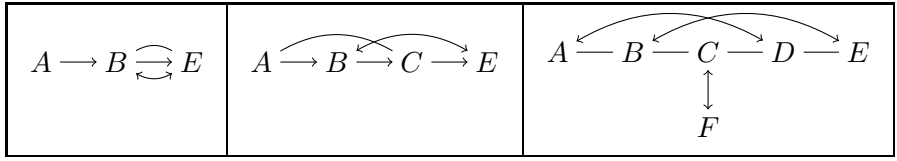

\section{Preliminaries}\label{sec:preliminaries}

In this section, we recall some concepts about graphical models. Unless otherwise stated, all the graphs and probability distributions in this paper are defined over a finite set $V$. The elements of $V$ are not distinguished from singletons. An ADMG $G$ is a graph with possibly directed, undirected and bidirected edges but without directed cycles, i.e. $G$ cannot have a subgraph $A \ra \ldots \ra A$. There may be up to three edges between any pair of nodes, but the edges must be different. Edges between a node and itself are not allowed. See Figure \ref{fig:example} for examples of ADMGs. Note that oADMGs are ADMGs without undirected edges, and aADMGs are ADMGs without bidirected edges.

Given an ADMG $G$, we represent with $A \oa B$ that $A \ra B$ or $A \aa B$ (or both) is in $G$. The parents of $X \subseteq V$ in $G$ are $Pa_G(X) = \{A | A \ra B$ is in $G$ with $B \in X \}$. The children of $X \subseteq V$ in $G$ are $Ch_G(X) = \{B | A \ra B$ is in $G$ with $A \in X \}$. The neighbors of $X \subseteq V$ in $G$ are $Ne_G(X) = \{A | A - B$ is in $G$ with $B \in X \}$. The spouses of $X \subseteq V$ in $G$ are $Sp_G(X) = \{A | A \aa B$ is in $G$ with $B \in X \}$. The descendants of $X \subseteq V$ in $G$ are $De_G(X) = \{B | A \ra \ldots \ra B$ is in $G$ with $A \in X$ or $B \in X \}$. The ancestors of $X \subseteq V$ in $G$ are $An_G(X) = \{A | A \ra \ldots \ra B$ is in $G$ with $B \in X$ or $A \in X \}$. Moreover, if $An_G(X) = X$ then we say that $X$ is an ancestral set. A route between a node $V_{1}$ and a node $V_{n}$ on $G$ is a sequence of (not necessarily distinct) nodes $V_{1}, \ldots, V_{n}$ such that $V_i$ and $V_{i+1}$ are adjacent in $G$ for all $1 \leq i < n$. We do not distinguish between the sequences $V_{1}, \ldots, V_{n}$ and $V_{n}, \ldots, V_{1}$, i.e. they represent the same route. If the nodes in the route are all distinct, then the route is called a path. The subgraph of $G$ induced by $X \subseteq V$, denoted as $G_X$, is the graph over $X$ that has all and only the edges in $G$ whose both ends are in $X$. Similarly, let $G^X$ denote the graph over $X \subseteq V$ constructed as follows: $A \oa B$ is in $G^X$ if and only if $A \oa B$ is in $G$, whereas $A - B$ is in $G^X$ if and only if $A - B$ is in $G$ or $A - V_1 - \ldots - V_n - B$ is in $G$ with $V_1, \ldots, V_n \notin X$.

\section{Separation Criteria}\label{sec:separation}

In this section, we introduce two equivalent separation criteria for ADMGs that define their semantics as a formalism to represent independence models. Specifically, a node $C$ on a path in an ADMG $G$ is said to be a collider on the path if $A \oa C \ao B$ or $A \oa C - B$ is a subpath. Moreover, the path is said to be connecting given $Z \subseteq V$ when
\begin{itemize}
\item every collider on the path is in $An_G(Z)$, and

\item every non-collider $C$ on the path is outside $Z$ unless $A - C - B$ is a subpath and $Pa_G(C) \setminus Z \neq \emptyset$ or $Sp_G(C) \neq \emptyset$.
\end{itemize}

Let $X$, $Y$ and $Z$ denote three disjoint subsets of $V$. When there is no path in $G$ connecting a node in $X$ and a node in $Y$ given $Z$, we say that $X$ is separated from $Y$ given $Z$ in $G$ and denote it as $X \ci_G Y | Z$. Note that this separation criterion generalizes the existing separation criteria for UGs, BGs, DAGs, AMP and MVR CGs, oADMGs and aADMGs. In other words, we can use the criterion above on all these families of graphical models.

Unlike in UGs, BGs, DAGs, and AMP and MVR CGs, two non-adjacent nodes in an ADMG are not necessarily separated. For example, $A \ci_G E | Z$ does not hold for any $Z$ in the ADMGs in Figure \ref{fig:example}. This drawback is shared by oADMGs \citep[p. 752]{EvansandRichardson2013}, summary graphs and MC graphs \citep[p. 1023]{RichardsonandSpirtes2002}, and ancestral graphs \citep[Section 3.7]{RichardsonandSpirtes2002}. For ancestral graphs, the problem can be solved by adding edges to the graph without altering the separations represented until every missing edge corresponds to a separation \citep[Section 5.1]{RichardsonandSpirtes2002}. A similar solution does not exist for ADMGs (we omit the details).

Finally, we present an alternative to the separation criterion introduced above. The alternative is easier to work with in some cases. The theorem below proves that both criteria are equivalent. Specifically, a node $C$ on a route in an ADMG $G$ is said to be a collider on the route if $A \oa C \ao B$ or $A \oa C - B$ is a subroute. Note that maybe $A = B$. Moreover, the route is said to be connecting given $Z \subseteq V$ when
\begin{itemize}
\item every collider on the route is in $Z$, and

\item every non-collider $C$ on the route is outside $Z$ unless $A - C - B$ is a subroute and $Sp_G(C) \neq \emptyset$.
\end{itemize}

Let $X$, $Y$ and $Z$ denote three disjoint subsets of $V$. When there is no route in $G$ connecting a node in $X$ and a node in $Y$ given $Z$, we say that $X$ is separated from $Y$ given $Z$ in $G$ and denote it as $X \ci_G Y | Z$.

\begin{theorem}\label{the:2}
Given $\alpha, \beta \in V$ and $Z \subseteq V \setminus ( \alpha \cup \beta )$, there is a path in an ADMG $G$ connecting $\alpha$ and $\beta$ given $Z$ if and only if there is a route in $G$ connecting $\alpha$ and $\beta$ given $Z$.
\end{theorem}

\begin{proof}
The only if part is trivial. To prove the if part, first replace every edge $A \aa B$ in $G$ with the subgraph $A \la \lambda_{AB} \ra B$, where $\lambda_{AB}$ is a newly created node. The result is an aADMG $G'$ over $V \cup \lambda$, where $\lambda$ denotes all the newly created nodes. Then, note that the route $\varrho$ in $G$ connecting $\alpha$ and $\beta$ given $Z$ can be transformed into a route $\varrho'$ in $G'$ connecting $\alpha$ and $\beta$ given $Z$ by simply replacing every edge $A \aa B$ in $\varrho$ with the subgraph $A \la \lambda_{AB} \ra B$. To see that $\varrho'$ is connecting, it may be worth noting that if $A - C - B$ is a subroute of $\varrho$ with $C \in Z$ and $Pa_{G}(C) \setminus Z = \emptyset$, then $Sp_G(C) \neq \emptyset$ for $\varrho$ to be connecting and, thus, $Pa_{G'}(C) \setminus Z \neq \emptyset$ since $\lambda_{CD} \in Pa_{G'}(C)$ for any $D \in Sp_G(C)$, and $\lambda_{CD} \notin Z$ since $Z \subseteq V$. Finally, note that $\varrho'$ can be transformed into a path $\rho'$ in $G'$ connecting $\alpha$ and $\beta$ given $Z$ \citep[Theorem 2]{Penna2016}, which can be transformed into a path $\rho$ in $G$ connecting $\alpha$ and $\beta$ given $Z$ by simply replacing every subpath $A \la \lambda_{AB} \ra B$ of $\rho'$ with the edge $A \aa B$. To see that $\rho$ is connecting, it may be worth noting that if $A - C - B$ is a subpath of $\rho'$ with $C \in Z$ and $Pa_{G'}(C) \setminus Z \neq \emptyset$, then $A - C - B$ is a subpath of $\rho$ with $Pa_{G}(C) \setminus Z \neq \emptyset$ or $Sp_{G}(C) \neq \emptyset$.
\end{proof}

\begin{table}
\caption{Magnification of an ADMG.}\label{tab:magnification}
\begin{center}
\begin{tabular}{|ll|}
\hline
1 & Set $G'=G$\\
2 & For each edge $A \aa B$ in $G$\\
3 & \hspace{0.3cm} Add the node $\lambda_{AB}$ to $G'$\\
4 & \hspace{0.3cm} Replace $A \aa B$ in $G'$ with the subgraph $A \la \lambda_{AB} \ra B$\\
5 & For each node $A$ in $G$\\
6 & \hspace{0.3cm} Add the node $\epsilon_A$ and the edge $\epsilon_A \ra A$ to $G'$\\
7 & For each edge $A - B$ in $G$\\
8 & \hspace{0.3cm} Replace $A - B$ in $G'$ with the edge $\epsilon_A - \epsilon_B$\\
\hline
\end{tabular}
\end{center}
\end{table}

\section{Causal Interpretation and Learning Algorithm}\label{sec:interpretationandlearning}

The contribution of this section is two-fold. First, it provides an intuitive causal interpretation of ADMGs as systems of structural equations with additive and correlated errors. Second, it describes an exact algorithm for learning ADMGs from observational and interventional data via answer set programming \citep{gelfond_1988,DBLP:journals/amai/Niemela99,DBLP:journals/ai/SimonsNS02}.

\subsection{Causal Interpretation}\label{sec:causal}

Let us assume that $V$ is normally distributed. In this section, we show that an ADMG $G$ can be interpreted as a system of structural equations with correlated errors. Specifically, the system includes an equation for each $A \in V$, which is of the form
\begin{equation}\label{eq:equation}
A = \sum_{B \in Pa_G(A)} \alpha_{AB} B + \sum_{B \in Sp_G(A)} \beta_{AB} \lambda_{AB} + \epsilon_A
\end{equation}
where $\alpha_{AB}$ and $\beta_{AB}$ denote linear coefficients, and $\lambda_{AB}$ and $\epsilon_A$ denote unobserved terms due to latent causes and errors, respectively. In other words, we divide the unobserved causes of $A$ into those shared with other observed variables (latent causes or confounders) and those exclusive of $A$ (errors). The undirected edges in $G$ indicate potential correlation between error terms. The latent causes and errors are represented implicitly in $G$. They can be represented explicitly by magnifying $G$ into the ADMG $G'$ as shown in Table \ref{tab:magnification}. The magnification basically consists in adding nodes for the unobserved terms $\lambda_{AB}$ and $\epsilon_A$ to $G$ and, then, connecting them appropriately. Figure \ref{fig:example2} shows an example. Note that Equation \ref{eq:equation} implies that every node $A \in V$ is determined by $Pa_{G'}(A)$. Likewise, $\epsilon_A$ is determined by $A \cup Pa_{G'}(A) \setminus \epsilon_A$, and $\lambda_{AB}$ is determined by $A \cup Pa_{G'}(A) \setminus \lambda_{AB}$. Let $\epsilon$ denote all the error nodes $\epsilon_A$ in $G'$, and let $\lambda$ denote all the latent causes $\lambda_{AB}$ in $G'$. Formally, we say that $A \in V \cup \lambda \cup \epsilon$ is determined by $Z \subseteq V \cup \lambda \cup \epsilon$ when $A \in Z$ or $A$ is a function of $Z$. We use $Dt(Z)$ to denote all the nodes that are determined by $Z$. From the point of view of the separations, that a node outside the conditioning set of a separation is determined by the conditioning set has the same effect as if the node were actually in the conditioning set. Bearing this in mind, it is not difficult to see that, as desired, $G$ and $G'$ represent the same separations over $V$. The following theorem formalizes this result.

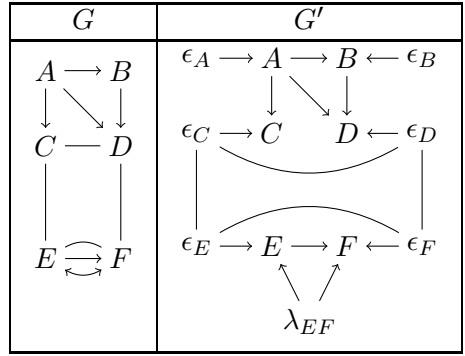
\begin{figure}
\centering
\begin{tabular}{|c|c|}
\hline
$G$&$G'$\\
\hline
\begin{tikzpicture}[inner sep=1mm]
\node at (0,0) (A) {$A$};
\node at (1,0) (B) {$B$};
\node at (0,-1) (C) {$C$};
\node at (1,-1) (D) {$D$};
\node at (0,-2.5) (E) {$E$};
\node at (1,-2.5) (F) {$F$};
\node at (0.5,-3.5) (LEF) {};
\path[->] (A) edge (B);
\path[->] (A) edge (C);
\path[->] (A) edge (D);
\path[->] (B) edge (D);
\path[-] (C) edge (D);
\path[-] (C) edge (E);
\path[-] (D) edge (F);
\path[-] (E) edge [bend left] (F);
\path[->] (E) edge (F);
\path[<->] (E) edge [bend right] (F);
\end{tikzpicture}
&
\begin{tikzpicture}[inner sep=1mm]
\node at (0,0) (A) {$A$};
\node at (1,0) (B) {$B$};
\node at (0,-1) (C) {$C$};
\node at (1,-1) (D) {$D$};
\node at (0,-2.5) (E) {$E$};
\node at (1,-2.5) (F) {$F$};
\node at (-1,0) (EA) {$\epsilon_A$};
\node at (2,0) (EB) {$\epsilon_B$};
\node at (-1,-1) (EC) {$\epsilon_C$};
\node at (2,-1) (ED) {$\epsilon_D$};
\node at (-1,-2.5) (EE) {$\epsilon_E$};
\node at (2,-2.5) (EF) {$\epsilon_F$};
\node at (0.5,-3.5) (LEF) {$\lambda_{EF}$};
\path[->] (EA) edge (A);
\path[->] (EB) edge (B);
\path[->] (EC) edge (C);
\path[->] (ED) edge (D);
\path[->] (EE) edge (E);
\path[->] (EF) edge (F);
\path[->] (A) edge (B);
\path[->] (A) edge (C);
\path[->] (A) edge (D);
\path[->] (B) edge (D);
\path[-] (EC) edge [bend right] (ED);
\path[-] (EC) edge (EE);
\path[-] (ED) edge (EF);
\path[-] (EE) edge [bend left] (EF);
\path[->] (E) edge (F);
\path[->] (LEF) edge (E);
\path[->] (LEF) edge (F);
\end{tikzpicture}\\
\hline
\end{tabular}\caption{Example of the magnification of an ADMG.}\label{fig:example2}
\end{figure}

\begin{theorem}\label{the:GG'}
Let $G$ denote an ADMG. Then, $X \ci_G Y | Z$ if and only if $X \ci_{G'} Y | Z$ for all $X$, $Y$ and $Z$ disjoint subsets of $V$.
\end{theorem}

\begin{proof}
Let $G'_4$ denote the graph $G'$ in Table \ref{tab:magnification} immediately after line 4. Note that $G'_4$ is an aADMG over $V \cup \lambda$. We know that $X \ci_{G'_4} Y | Z$ if and only if $X \ci_{G'} Y | Z$ \citep[Theorem 9]{Penna2016}. Therefore, it suffices to show that every path in $G$ connecting $\alpha$ and $\beta$ given $Z$ can be transformed into a path in $G'_4$ connecting $\alpha$ and $\beta$ given $Z$ and vice versa, with $\alpha, \beta \in V$ and $Z \subseteq V \setminus ( \alpha \cup \beta )$. This can be proven in much the same way as Theorem \ref{the:2}. Specifically, a path $\rho$ in $G$ connecting $\alpha$ and $\beta$ given $Z$ can be transformed into a path $\rho'$ in $G'_4$ connecting $\alpha$ and $\beta$ given $Z$ by simply replacing every edge $A \aa B$ in $\rho$ with the subgraph $A \la \lambda_{AB} \ra B$. Finally, a path $\rho'$ in $G'_4$ connecting $\alpha$ and $\beta$ given $Z$ can be transformed into a path $\rho$ in $G$ connecting $\alpha$ and $\beta$ given $Z$ by simply reversing the previous transformation.
\end{proof}

Let $\lambda \sim \mathcal{N}(0, \Lambda)$ such that $\Lambda$ is diagonal, and $\epsilon \sim \mathcal{N}(0, \Sigma)$ such that $(\Sigma^{-1})_{\epsilon_A,\epsilon_B} = 0$ if $\epsilon_A - \epsilon_B$ is not in $G'$. Then, $G$ can be interpreted as a system of structural equations of the form of Equation \ref{eq:equation} whose errors are correlated as follows
\begin{equation}\label{eq:equation2}
covariance(\epsilon_A, \epsilon_B) = \Sigma_{\epsilon_A,\epsilon_B}
\end{equation}
for all $A, B \in V$. The next two theorems confirm that this causal interpretation of ADMGs works as intended. Let $X$, $Y$ and $Z$ denote three disjoint subsets of $V$. Hereinafter, we represent by $X \ci_p Y | Z$ that $X$ and $Y$ are conditionally independent given $Z$ in a probability distribution $p$.

\begin{theorem}\label{the:gaussian}
Let $G$ and $p$ denote an ADMG and a probability distribution over $V$. If $p$ is specified by Equations \ref{eq:equation} and \ref{eq:equation2}, then it is Gaussian.
\end{theorem}

\begin{proof}
For each edge $A \aa B$ in $G$, add the node $\lambda_{AB}$ to $G$. Then, replace every edge $A \aa B$ in $G$ with the subgraph $A \la \lambda_{AB} \ra B$. Note that $G$ is now an aADMG over $V \cup \lambda$. Moreover, recall that $\lambda \sim \mathcal{N}(0, \Lambda)$. Then, add the equation 
\begin{equation}\label{eq:equation3}
\lambda_{AB} = \epsilon'_{AB}
\end{equation}
and let $\epsilon' \sim \mathcal{N}(0, \Lambda)$, where $\epsilon'$ denotes all the newly created error terms $\epsilon'_{AB}$. Then, every probability distribution $p(V \cup \lambda)$ specified by Equations \ref{eq:equation}-\ref{eq:equation3} is Gaussian \citep[Theorem 10]{Penna2016}, which implies the desired result.
\end{proof}

It is worth mentioning that the opposite of the theorem above is not true. This negative result is inherited from oADMGs, for which there are Gaussian probability distributions over $V$ that cannot be specified by Equations \ref{eq:equation} and \ref{eq:equation2} \citep[p. 1019]{RichardsonandSpirtes2002}.

\begin{theorem}\label{the:markovian}
Let $G$ and $p$ denote an ADMG and a probability distribution over $V$. If $p$ is specified by Equations \ref{eq:equation} and \ref{eq:equation2}, then $X \ci_G Y | Z$ implies that $X \ci_p Y | Z$ for all $X$, $Y$ and $Z$ disjoint subsets of $V$.
\end{theorem}

\begin{proof}
Transform $G$ into an aADMG over $V \cup \lambda$ as shown in the proof of Theorem \ref{the:gaussian}. Then, $X \ci_G Y | Z$ implies that $X \ci_{p(V \cup \lambda)} Y | Z$ \citep[Theorem 11]{Penna2016}, which implies the desired result.
\end{proof}

A more intuitive account of the causal interpretation of ADMGs introduced above is as follows. We interpret the edge $A \ra B$ as $A$ being a potential cause of $B$. We interpret the edge $A \aa B$ as $A$ and $B$ being potentially marginally dependent due to an unobserved common cause $\lambda_{AB}$, i.e. a confounder. The unobserved causes of the node $A$ that are not shared with any other node are grouped into an error term $\epsilon_A$. We interpret the edge $A - B$ as $\epsilon_A$ and $\epsilon_B$ being potentially conditionally dependent given the rest of the error terms. This causal interpretation of ADMGs generalizes that of the oADMGs and aADMGs. Recall however that the noise in the oADMGs is not necessarily additive normal.

\begin{table}
\caption{Intervention of $X$ on an ADMG.}\label{tab:intervention}
\begin{center}
\begin{tabular}{|ll|}
\hline
1 & Delete from $G$ all the edges $A \oa B$ with $B \in X$\\
2 & For each path $A - V_1 - \ldots - V_n - B$ in $G$ with $A, B \notin X$ and $V_1, \ldots, V_n \in X$\\
3 & \hspace{0.3cm} Add the edge $A - B$ to $G$\\
4 & Delete from $G$ all the edges $A - B$ with $B \in X$\\
\hline
\end{tabular}
\end{center}
\end{table}

Given the above causal interpretation of an ADMG $G$ and assuming autonomous causal relationships (i.e. external changes to one does not affect the others), intervening on $X \subseteq V$ so that $X$ is no longer under the influence of its usual causes amounts to replacing the right-hand side of the equations for the random variables in $X$ with expressions that do not involve their usual causes \citep[Section 3.2]{Pearl2009}. For simplicity, we only consider interventions that set $X$ to a fixed value $x$, which then corresponds to modifying the system of structural equations by replacing the equation for every $X_i \in X$ with the equation $X_i=x_i$, where $x_i$ is the value of $X_i$ that is consistent with $x$. Graphically, it amounts to modifying $G$ as shown in Table \ref{tab:intervention}. Line 1 is shared with an intervention on an oADMG. Lines 2-4 are best understood in terms of the magnified ADMG $G'$: They correspond to marginalizing the error nodes associated with the nodes in $X$ out of $G'_\epsilon$, the UG that represents the correlation structure of the error nodes. In other words, lines 2-4 replace $G'_\epsilon$ with $(G'_\epsilon)^{\epsilon \setminus \epsilon_X}$, the marginal graph of $G'_\epsilon$ over $\epsilon \setminus \epsilon_X$. This makes sense since $\epsilon_X$ is no longer associated with $X$ due to the intervention and, thus, we may want to marginalize it out because it is unobserved. This is exactly what lines 2-4 imply. Note that the ADMG after the intervention and the magnified ADMG after the intervention represent the same separations over $V$, by Theorem \ref{the:GG'}. This treatment of interventions on ADMGs generalizes the treatment for oADMGs and aADMGs \citep{Pearl2009,Penna2016}.

We can also extend the separation criteria for ADMGs to account for interventions. Specifically, let $X \ci_{G_{\wh{W}}} Y | Z, W$ denote that $X$ is separated from $Y$ given $Z$ in an ADMG $G$ after having intervened on $W$, where $X$, $Y$, $Z$ and $W$ are disjoint subsets of $V$. Likewise, let $X \ci_{p_{\wh{W}}} Y | Z, W$ represent that $X$ and $Y$ are conditionally independent given $Z$ in a probability distribution $p$ after having intervened on $W$. The corollary below follows from Theorem \ref{the:markovian}, and provides further evidence that the causal interpretation of ADMGs introduced above works as intended.

\begin{corollary}\label{cor:markovian2}
Let $G$ and $p$ denote an ADMG and a probability distribution over $V$. If $p$ is specified by Equations \ref{eq:equation} and \ref{eq:equation2}. Then, $X \ci_{G_{\wh{W}}} Y | Z, W$ implies that $X \ci_{p_{\wh{W}}} Y | Z, W$ for all $X$, $Y$, $Z$ and $W$ disjoint subsets of $V$.
\end{corollary}

Recall from Section \ref{sec:separation} that two non-adjacent nodes in an ADMG $G$ are not necessarily separated. This is not true when interventions are considered, because $A \ci_{G_{\wh{W}}} B | W$ with $W = V \setminus \{A, B\}$ for all non-adjacent nodes $A$ and $B$ of $G$. Therefore, some missing edges in $G$ convey information about the observational regime, and some others about the interventional regime.

Finally, note that Equations \ref{eq:equation} and \ref{eq:equation2} specify each node as a linear function of its parents with additive normal noise. The equations can be generalized to nonlinear or nonparametric functions as long as the noise remains additive normal. That is, for any $A \in V$
\[
A = f(Pa_{G'}(A) \setminus \epsilon_A) + \epsilon_A
\]
with $\epsilon \sim \mathcal{N}(0, \Sigma)$ such that $(\Sigma^{-1})_{\epsilon_A,\epsilon_B} = 0$ if $\epsilon_A - \epsilon_B$ is not in $G'$. That the noise is additive normal ensures that $\epsilon_A$ is determined by $A \cup Pa_{G'}(A) \setminus \epsilon_A$, which is needed for Theorem \ref{the:GG'} to remain valid which, in turn, is needed for Theorem \ref{the:markovian} and Corollary \ref{cor:markovian2} to remain valid.

\begin{table}
\caption{ASP encoding of the learning algorithm.}\label{tab:asp}
\begin{center}
\tiny
\fbox{
\begin{minipage}{0.7\textwidth}
\begin{verbatim}
% input predicates
% nodes(N): N is the number of nodes 
% set(X): X is the index of a set of nodes
% dep(X,Y,C,I,W) (resp. indep(X,Y,C,I,W)): the nodes X and Y are dependent (resp.
%                                          independent) given the set of nodes C
%                                          after having intervened on the node I

% nodes
node(X) :- nodes(N), X=1..N.                                               % rule 1

% edges
{ line(X,Y,0) } :- node(X), node(Y), X != Y.                               %      2
{ arrow(X,Y,0) } :- node(X), node(Y), X != Y.
{ biarrow(X,Y,0) } :- node(X), node(Y), X != Y.                            %      4
line(X,Y,I) :- line(X,Y,0), node(I), X != I, Y != I, I > 0.                %      5
line(X,Y,I) :- line(X,I,0), line(I,Y,0), node(I), X != Y, I > 0.
arrow(X,Y,I) :- arrow(X,Y,0), node(I), Y != I, I > 0.
biarrow(X,Y,I) :- biarrow(X,Y,0), node(I), X != I, Y != I, I > 0.          %      8
line(X,Y,I) :- line(Y,X,I).                                                %      9
:- arrow(X,Y,I), arrow(Y,X,I).
biarrow(X,Y,I) :- biarrow(Y,X,I).                                          %      11

% directed acyclity
ancestor(X,Y) :- arrow(X,Y,0).                                             %      12
ancestor(X,Y) :- ancestor(X,Z), ancestor(Z,Y).
:- ancestor(X,Y), arrow(Y,X,0).                                            %      14

% set membership
inside_set(X,C) :- node(X), set(C), 2**(X-1) & C != 0.                     %      15
outside_set(X,C) :- node(X), set(C), 2**(X-1) & C = 0.                     %      16

% end_line/head/tail(X,Y,C,I) means that there is a connecting route 
% from X to Y given C that ends with a line/arrowhead/arrowtail

% single edge route
end_line(X,Y,C,I) :- line(X,Y,I), outside_set(X,C).                        %      17
end_head(X,Y,C,I) :- arrow(X,Y,I), outside_set(X,C).
end_head(X,Y,C,I) :- biarrow(X,Y,I), outside_set(X,C).
end_tail(X,Y,C,I) :- arrow(Y,X,I), outside_set(X,C).

% connection through non-collider
end_line(X,Y,C,I) :- end_line(X,Z,C,I), line(Z,Y,I), outside_set(Z,C).
end_line(X,Y,C,I) :- end_line(X,Z,C,I), line(Z,Y,I), biarrow(Z,W,I).
end_line(X,Y,C,I) :- end_tail(X,Z,C,I), line(Z,Y,I), outside_set(Z,C).
end_head(X,Y,C,I) :- end_line(X,Z,C,I), arrow(Z,Y,I), outside_set(Z,C).
end_head(X,Y,C,I) :- end_head(X,Z,C,I), arrow(Z,Y,I), outside_set(Z,C).
end_head(X,Y,C,I) :- end_tail(X,Z,C,I), arrow(Z,Y,I), outside_set(Z,C).
end_head(X,Y,C,I) :- end_tail(X,Z,C,I), biarrow(Z,Y,I), outside_set(Z,C).
end_tail(X,Y,C,I) :- end_tail(X,Z,C,I), arrow(Y,Z,I), outside_set(Z,C).

% connection through collider
end_line(X,Y,C,I) :- end_head(X,Z,C,I), line(Z,Y,I), inside_set(Z,C).
end_head(X,Y,C,I) :- end_line(X,Z,C,I), biarrow(Z,Y,I), inside_set(Z,C).
end_head(X,Y,C,I) :- end_head(X,Z,C,I), biarrow(Z,Y,I), inside_set(Z,C).
end_tail(X,Y,C,I) :- end_line(X,Z,C,I), arrow(Y,Z,I), inside_set(Z,C).
end_tail(X,Y,C,I) :- end_head(X,Z,C,I), arrow(Y,Z,I), inside_set(Z,C).     %      33

% derived non-separations
con(X,Y,C,I) :- end_line(X,Y,C,I), X != Y, outside_set(Y,C).               %      34
con(X,Y,C,I) :- end_head(X,Y,C,I), X != Y, outside_set(Y,C).
con(X,Y,C,I) :- end_tail(X,Y,C,I), X != Y, outside_set(Y,C).             
con(X,Y,C,I) :- con(Y,X,C,I).                                              %      37

% satisfy all dependences
:- dep(X,Y,C,I,W), not con(X,Y,C,I).                                       %      38

% maximize the number of satisfied independences
:~ indep(X,Y,C,I,W), con(X,Y,C,I). [W,X,Y,C,I]                             %      39

% minimize the number of lines/arrows
:~ line(X,Y,0), X < Y. [1,X,Y,1]                                           %      40
:~ arrow(X,Y,0). [1,X,Y,2]
:~ biarrow(X,Y,0), X < Y. [1,X,Y,3]                                        %      42

% show results
#show. 
#show line(X,Y) : line(X,Y,0), X < Y.
#show arrow(X,Y) : arrow(X,Y,0).
#show biarrow(X,Y) : biarrow(X,Y,0), X < Y.
\end{verbatim}
\end{minipage}
}
\end{center}
\end{table}

\subsection{Learning Algorithm}\label{sec:learning}

In this section, we introduce an exact algorithm for learning ADMGs from observational and interventional data via answer set programming (ASP), which is a declarative constraint satisfaction paradigm that is well-suited for solving computationally hard combinatorial problems \citep{gelfond_1988,DBLP:journals/amai/Niemela99,DBLP:journals/ai/SimonsNS02}. ASP represents constraints in terms of first-order logical rules. Therefore, when using ASP, the first task is to model the problem at hand in terms of rules so that the set of solutions implicitly represented by the rules corresponds to the solutions of the original problem. One or multiple solutions to the original problem can then be obtained by invoking an off-the-shelf ASP solver on the constraint declaration. Each rule in the constraint declaration is of the form \verb|head :- body|. The head contains an atom, i.e. a fact. The body may contain several literals, i.e. negated and non-negated atoms. Intuitively, the rule is a justification to derive the head if the body is true. The body is true if its non-negated atoms can be derived, and its negated atoms cannot. A rule with only the head is an atom. A rule without the head is a hard-constraint, meaning that satisfying the body results in a contradiction. Soft-constraints are encoded as rules of the form \verb|:~ body. [W]|, meaning that satisfying the body results in a penalty of $W$ units. The ASP solver returns the solutions that meet the hard-constraints and minimize the total penalty due to the soft-constraints. In this work, we use the ASP solver \verb|clingo| \citep{DBLP:journals/aicom/GebserKKOSS11}, whose underlying algorithms are based on state-of-the-art Boolean satisfiability solving techniques \citep{DBLP:series/faia/2009-185}.

\begin{figure}
\caption{ASP encoding of the (in)dependences in the domain.}\label{tab:asp2}
\begin{center}
\tiny
\fbox{
\begin{minipage}{0.35\textwidth}
\input{data.tex}
\end{minipage}
}
\end{center}
\end{figure}

Table \ref{tab:asp} shows the ASP encoding of our learning algorithm. The predicate \verb|node(X)| in rule 1 represents that $X$ is a node. The predicates \verb|line(X,Y,I)|, \verb|arrow(X,Y,I)| and \linebreak \verb|biarrow(X,Y,I)| represent that there is an undirected, directed and bidirected edge from the node $X$ to the node $Y$ after having intervened on the node $I$. The observational regime corresponds to $I=0$. The rules 2-4 encode a non-deterministic guess of the edges for the observational regime, which means that the ASP solver will implicitly consider all possible graphs during the search, hence the exactness of the search. The edges under the observational regime are used in the rules 5-8 to define the edges in the graph after having intervened on $I$, following the description in Section \ref{sec:interpretationandlearning}. Therefore, the algorithm assumes continuous random variables and additive normal noise when the input contains interventions. The random variables do not need to be normally distributed though, as discussed at the end of Section \ref{sec:causal}. The algorithm makes no such assumption when the input consists of just observations. The rules 9-11 enforce the fact that bidirected and undirected edges are symmetric and that there is at most one directed edge between two nodes. The predicate \verb|ancestor(X,Y)| represents that the node $X$ is an ancestor of the node $Y$. The rules 12-14 enforce that the graph has no directed cycles. The predicates in the rules 15-16 represent whether a node $X$ is or is not in a set of nodes $C$. The rules 17-33 encode the alternative separation criterion introduced in Section \ref{sec:separation}. The predicate \verb|con(X,Y,C,I)| in rules 34-37 represents that there is a connecting route between the node $X$ and the node $Y$ given the set of nodes $C$ after having intervened on the node $I$. The rule 38 enforces that each dependence in the input must correspond to a connecting route. The rule 39 represents that each independence in the input that is not represented implies a penalty of $W$ units. The rules 40-42 represent a penalty of 1 unit per edge. Other penalty rules can be added similarly.

Table \ref{tab:asp2} illustrates with an example how to encode the (in)dependences in the probability distribution at hand, e.g. as determined from some available data. Specifically, the predicate \verb|nodes(3)| represents that there are three nodes in the domain at hand, and the predicate \verb|set(0..7)| represents that there are eight sets of nodes, indexed from 0 (empty set) to 7 (full set). The predicate \verb|indep(X,Y,C,I,W)| (respectively \verb|dep(X,Y,C,I,W)|) represents that the nodes $X$ and $Y$ are conditionally independent (respectively dependent) given the node set index $C$ after having intervened on the node $I$. Observations correspond to $I=0$. The penalty for failing to represent an (in)dependence is $W$. The penalty for failing to represent a dependence is actually superfluous in our algorithm, since rule 38 in Table \ref{tab:asp} enforces that all the dependences in the input are represented by the graph in the output. Note also that it suffices to specify all the (in)dependences between pair of nodes, because these identify uniquely the rest of the independences in the probability distribution \citep[Lemma 2.2]{Studeny2005}. Note also that we do not assume that the probability distribution at hand is faithful to some ADMG or that it satisfies the composition property, as it is the case in most heuristic learning algorithms.

By calling the ASP solver with the encodings of the learning algorithm and the (in)depen- dences in the domain, the solver will essentially perform an exhaustive search over the space of graphs, and will output the graphs with the smallest penalty. Specifically, when only the observations are used (i.e. the last 15 lines of Table \ref{tab:asp2} are removed), the learning algorithm finds 104 optimal models, including one UG, one BG, six DAGs, 13 AMP CGs, 13 MVR CGs, 37 original ADMGs, and 37 alternative ADMGs. When all the observations and interventions available are used, the learning algorithm finds two optimal models. These are the models on the left and center of Figure \ref{fig:example3b}. This is the expected result given the last 15 lines in Table \ref{tab:asp2}. The rightmost model in Figure \ref{fig:example3b} is not in the output because, although it is indistinguishable from the other two given the observations and interventions in the input, it has more edges and thus receives a larger penalty, which makes it suboptimal.

It is worth mentioning that the example above is just illustrative and, thus, we have made use of an oracle to detect (in)dependencies in the domain at hand. In reality, (in)dependencies are detected on the basis of a finite sample and, thus, conflicts between them may exist. A solution for conflict resolution within the ASP framework has been proposed by \cite{Hyttinenetal.2014}. This solution has moreover been included in the ASP algorithm for learning LWF CGs proposed by \cite{Sonntagetal.2015}. Since the latter has been the basis for our learning algorithm in Table \ref{tab:asp}, we expect that the conflict resolution of \cite{Hyttinenetal.2014} can also be adapted to our algorithm for enhanced performance in practice.

\begin{figure}
\centering
\begin{tabular}{|c|c|c|}
\hline
\begin{tikzpicture}[inner sep=1mm]
\node at (0,0) (A) {$1$};
\node at (1,0) (B) {$2$};
\node at (2,0) (D) {$3$};
\path[->] (A) edge (B);
\path[->] (B) edge (D);
\path[-] (B) edge [bend left] (D);
\end{tikzpicture}
&
\begin{tikzpicture}[inner sep=1mm]
\node at (0,0) (A) {$1$};
\node at (1,0) (B) {$2$};
\node at (2,0) (D) {$3$};
\path[->] (A) edge (B);
\path[->] (B) edge (D);
\path[<->] (B) edge [bend right] (D);
\end{tikzpicture}
&
\begin{tikzpicture}[inner sep=1mm]
\node at (0,0) (A) {$1$};
\node at (1,0) (B) {$2$};
\node at (2,0) (D) {$3$};
\path[->] (A) edge (B);
\path[->] (B) edge (D);
\path[-] (B) edge [bend left] (D);
\path[<->] (B) edge [bend right] (D);
\end{tikzpicture}
\\
\hline
\end{tabular}\caption{ADMGs that represent the (in)dependences in the domain.}\label{fig:example3b}
\end{figure}
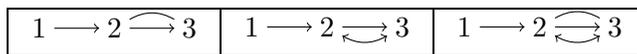

Finally, the ASP code in Table \ref{tab:asp} can easily be modified to learn some subfamilies of ADMGs such as
\begin{itemize}
\item original ADMGs by adding \verb|:- line(X,Y,0).|

\item alternative ADMGs by adding \verb|:- biarrow(X,Y,0).|

\item AMP CGs by adding \verb|:- biarrow(X,Y,0).|, \verb|:- line(X,Y,0), arrow(X,Y,0).| and \verb|ancestor(X,Y) :- line(X,Y,0).|

\item MVR CGs by adding \verb|:- line(X,Y,0).|, \verb|:- biarrow(X,Y,0), arrow(X,Y,0).| and \verb|ancestor(X,Y) :- biarrow(X,Y,0).|

\item DAGs by adding \verb|:- line(X,Y,0).| and \verb|:- biarrow(X,Y,0).|

\item UGs by adding \verb|:- arrow(X,Y,0).| and \verb|:- biarrow(X,Y,0).|

\item BGs by adding \verb|:- arrow(X,Y,0).| and \verb|:- line(X,Y,0).|
\end{itemize}

\section{Causal Effect Identification}\label{sec:identification}

This section presents graphical criteria for causal effect identification in aADMGs. Some criteria are based on a calculus similar to the {\it do}-calculus by \citet{Pearl2009}, and some on a decomposition of the distribution over $V$ similar to the Q-decomposition by \citet{TianandPearl2002a,TianandPearl2002b}.

\subsection{Calculus-Based Causal Effect Identification}\label{sec:calculus}

In this section, we present a new calculus for causal effect identification from aADMGs. The calculus consists of three rules that may transform a causal effect expression into an expression that can be computed from observed quantities. 

\begin{theorem}\label{the:calculus}
Let $G$ be an aADMG. Let $X$, $Y$, $Z$ and $W$ be disjoint subsets of variables. Then, we have the following rules.
\begin{itemize}
\item Rule 1 (insertion/deletion of observations):
\[
p(Y| \wh{X}, Z, W) = p(Y| \wh{X}, W) \text{ if } Y \ci_{G_{\wh{X}}} Z | X, W.
\]
\item Rule 2 (intervention/observation exchange): 
\[
p(Y| \wh{X}, \wh{Z}, W) = p(Y| \wh{X}, Z, W) \text{ if } Y \ci_{G_{\wh{X} \underrightarrow{Z}}} Z | X, W
\]
where $G_{\wh{X} \underrightarrow{Z}}$ denotes the graph obtained from $G_{\wh{X}}$ by deleting all directed edges out of $Z$.
\item Rule 3 (insertion/deletion of interventions): 
\[
p(Y| \wh{X}, \wh{Z}, W) = p(Y| \wh{X}, W) \text{ if } Y \ci_{G_{\wh{X} \overline{\overrightarrow{Z(W)}}}} Z | X, W
\]
where $Z(W)$ denotes the nodes in $Z$ that are not ancestors of $W$ in $G_{\wh{X}}$, and $G_{\wh{X} \overline{\overrightarrow{Z(W)}}}$ denotes the graph obtained from $G_{\wh{X}}$ by deleting all directed and undirected edges into $Z(W)$.
\end{itemize}
\end{theorem}

\begin{proof}
Rule 1 follows from Corollary \ref{cor:markovian2}. The antecedent of rule 2 implies that the only paths between $Y$ and $Z$ in $G_{\wh{X}}$ that are not blocked by $W \cup X$ are those that reach $Z$ through its children. Following \citet[p. 686]{Pearl1995}, we transform $G_{\wh{X}}$ into $G'_{\wh{X}}$ by adding a variable $F_C$ and an edge $F_C \ra C$ for all $C \in Z$. The domain of $F_C$ is the same as that of $C$ plus a state labeled {\it idle}: $F_C = c$ corresponds an intervention that sets $C = c$, whereas $F_C = {\it idle}$ represents that $C$ is observed rather than intervened upon. Then, the only paths between $Y$ and $F_Z$ in $G'_{\wh{X}}$ that are not blocked by $W \cup X$ are those that reach $F_Z$ through the children of $Z$ and, thus, they are blocked by $Z$. Then, observing $Z$ cannot be distinguished from intervening on $Z$ and, thus, the consequent of rule 2 holds.

The antecedent of rule 3 implies that the only paths between $Y$ and $Z$ in $G_{\wh{X}}$ that are not blocked by $W \cup X$ are those that reach $Z$ through the parents or neighbors of $Z(W)$. Then, there is no path between $Y$ and $F_Z$ in $G'_{\wh{X}}$ that is not blocked by $W \cup X$. Then, intervening on $Z$ is irrelevant and, thus, the consequent of rule 3 holds.
\end{proof}

We illustrate the application of the previous theorem with the aADMG in Figure \ref{fig:example1}. Specifically,
\[
p(B | \wh{A}) = \sum_c p(B | \wh{A}, c) p(c | \wh{A}) = \sum_c p(B | \wh{A}, c) p(c) = \sum_c p(B | A, c) p(c)
\]
where the first equality is due to marginalization, the second due to rule 3, and the third due to rule 2.

Since producing $G_{\wh{X}}$ may be a bit involved, the antecedents of the rules can be simplified as follows.

\begin{theorem}
Let $G$ be an aADMG. Let $X$, $Y$, $Z$ and $W$ be disjoint subsets of variables. Then, we have the following rules.
\begin{itemize}
\item Rule 1 (insertion/deletion of observations):
\[
p(Y| \wh{X}, Z, W) = p(Y| \wh{X}, W) \text{ if } Y \ci_{G_{\overrightarrow{\underrightarrow{X}}}} Z | W
\]
where $G_{\overrightarrow{\underrightarrow{X}}}$ denotes the graph obtained from $G$ by deleting all directed edges in and out of $X$.
\item Rule 2 (intervention/observation exchange): 
\[
p(Y| \wh{X}, \wh{Z}, W) = p(Y| \wh{X}, Z, W) \text{ if } Y \ci_{G_{\overrightarrow{\underrightarrow{X}} \underrightarrow{Z}}} Z | W
\]
where $G_{\overrightarrow{\underrightarrow{X}} \underrightarrow{Z}}$ denotes the graph obtained from $G$ by deleting all directed edges in and out of $X$ and out of $Z$.
\item Rule 3 (insertion/deletion of interventions): 
\[
p(Y| \wh{X}, \wh{Z}, W) = p(Y| \wh{X}, W) \text{ if } Y \ci_{G_{\overrightarrow{\underrightarrow{X}} \overline{\overrightarrow{Z(W)}}}} Z | W
\]
where $Z(W)$ denotes the nodes in $Z$ that are not ancestors of $W$ in $G_{\overrightarrow{\underrightarrow{X}}}$, and $G_{\overrightarrow{\underrightarrow{X}} \overline{\overrightarrow{Z(W)}}}$ denotes the graph obtained from $G_{\overrightarrow{\underrightarrow{X}}}$ by deleting all undirected and directed edges into $Z(W)$.
\end{itemize}
\end{theorem}

\begin{proof}
We prove that the antecedents of rule 1 in this theorem and in Theorem \ref{the:calculus} are equivalent. If $Y \nci_{G_{\wh{X}}} Z | X, W$ then there is a connecting path given $W \cup X$ that contains no node in $X$, since the nodes in X can only have outgoing directed edges in $G_{\wh{X}}$ and thus the path would be blocked. The path can be transformed into a path in $G_{\overrightarrow{\underrightarrow{X}}}$ by simply undoing lines 2 and 3 in Table \ref{tab:intervention}. The resulting path is clearly connecting given $W$. 

If $Y \nci_{G_{\overrightarrow{\underrightarrow{X}}}} Z | W$ then there is a connecting path given $W$. Note that the nodes in $X$ only participate in undirected edges in the path, since they do not participate in any directed edge in $G_{\overrightarrow{\underrightarrow{X}}}$. Then, the nodes in $X$ only appear in subpaths of the form $A - X_1 - \ldots - X_i - \ldots - X_k - B$ with $X_1, \ldots, X_k \in X$ and $A, B \notin X$. Then, the path clearly results in a path in $G_{\wh{X}}$ that is connecting given $W \cup X$, by lines 2 and 3 in Table \ref{tab:intervention}.

Proving the equivalence of the antecedents for rules 2 and 3 can be done similarly. In rule 3, note that a node in $Z$ is an ancestor of $W$ in $G_{\wh{X}}$ if and only if it is an ancestor of $W$ in $G_{\overrightarrow{\underrightarrow{X}}}$.
\end{proof}

We do not currently have a systematic way of deciding whether there exists a sequence of rules for identifying a given causal effect. The following theorems characterize graphically three cases where such a sequence exists.

\begin{theorem}\label{the:backdoor}
A set of variables $W \cup Z$ satisfies the back-door criterion relative to an ordered pair of variables $(X, Y)$ in an aADMG $G$ if
\begin{enumerate}
\item $W \cup Z$ contains no descendant of $X$ in $G$, and

\item $W \cup Z$ blocks all non-directed paths in $G$ from $X$ to $Y$. 
\end{enumerate}
Moreover, 
\[
p(Y | \wh{X}, W) = \sum_z p(Y | X, W, z) p(z | W).
\]
\end{theorem}

\begin{proof}
By condition 2, the only paths between $X$ and $Y$ that are not blocked given $W \cup Z$ are those that reach $X$ through its children. Then, $p(Y | \wh{X}, W, z) = p(Y | X, W, z)$ by rule 2. By condition 1, the only paths between $X$ and $Z$ that are not blocked given $W$ are those that reach $X$ through its parents or neighbors. Then, $p(z | \wh{X}, W) = p(z | W)$ by rule 3. Finally, simply replace the previous expressions in
\[
p(Y | \wh{X}, W) = \sum_z p(Y | \wh{X}, W, z) p(z | \wh{X}, W).
\]
\end{proof}

Note that $p(B | \wh{A})$ in the example in Figure \ref{fig:example1} can be identified with the help of the previous theorem ($X=A$, $Y=B$, $W = \emptyset$ and $Z=C$).

\begin{theorem}
A set of variables $W \cup Z$ satisfies the front-door criterion relative to an ordered pair of variables $(X, Y)$ in an aADMG $G$ if 
\begin{enumerate}
\item $W$ contains no descendant of $X$ in $G$,

\item $Z$ blocks all directed paths in $G$ from $X$ to $Y$,

\item $W$ satisfies the back-door criterion in $G$ relative to $(X,Z)$, and

\item $W \cup X$ satisfies the back-door criterion in $G$ relative to $(Z,Y)$.
\end{enumerate}
Moreover,
\[
p(Y | \wh{X}, W) = \sum_z p(z | X, W) \sum_{x'} p(Y | z, W, x') p(x' | W)
\]
\end{theorem}

\begin{proof}
By condition 4, the only paths between $Z$ and $Y$ that are not blocked by $W \cup X$ are those that reach $Z$ through its children. Then, $p(Y | \wh{X}, W, z) = p(Y | \wh{X}, W, \wh{z})$ by rule 2. By conditions 1 and 2, the only paths between $X$ and $Y$ that are not blocked by $W \cup Z$ are those that reach $X$ through its parents or neighbors and, thus, $p(Y | \wh{X}, W, \wh{z}) = p(Y | \wh{z}, W)$ by rule 3. To see it, note that any path reaching $X$ through its children must be of the form $X \ra \ldots \ra Z_i \ldots Y$ where $Z_i \in Z$ is a collider on the path. However, the subpath between $Z_i$ and $Y$ contradicts condition 4. Moreover, note that conditions 3 and 4 together with Theorem \ref{the:backdoor} imply that
\[
p(Z | \wh{X}, W) = p(Z | X, W)
\]
and
\[
p(Y | \wh{Z}, W) = \sum_{x'} p(Y | Z, W, x') p(x' | W).
\]
Finally, replace the previous expressions in
\[
p(Y | \wh{X}, W) = \sum_z p(Y | \wh{X}, W, z) p(z | \wh{X}, W) = \sum_z p(Y | \wh{X}, W, \wh{z}) p(z | \wh{X}, W)
\]
\[
= \sum_z p(Y | \wh{z}, W) p(z | \wh{X}, W).
\]
\end{proof}

Figure \ref{fig:example3} shows a DAG that induces an aADMG from which $p(D | \wh{A}, C)$ can be identified with the help of the previous theorem ($X=A$, $Y=D$, $W=C$ and $Z=B$). Note that Theorem \ref{the:backdoor} is not applicable to this example. Note also that the DAG induces an oADMG from which the causal effect is not identifiable \citep[p. 94]{Pearl2009}.

\begin{figure}
\begin{center}
\begin{tabular}{|c|c|c|}
\hline
DAG&oADMG&aADMG\\
\hline
\begin{tikzpicture}[inner sep=1mm]
\node at (0,0) (A) {$A$};
\node at (2,0) (B) {$B$};
\node at (1,1) (C) {$C$};
\node at (3,0) (D) {$D$};
\node at (0,3) (UA) {$U_A$};
\node at (2,1) (UB) {$U_B$};
\node at (1,2) (UC) {$U_C$};
\node at (3,1) (UD) {$U_D$};
\path[->] (A) edge (B);
\path[->] (B) edge (D);
\path[->] (UA) edge (A);
\path[->] (UB) edge (B);
\path[->] (UC) edge (C);
\path[->] (UD) edge (D);
\path[->] (UA) edge (UC);
\path[->] (UC) edge (UB);
\path[->] (UA) edge (UD);
\end{tikzpicture}
&
\begin{tikzpicture}[inner sep=1mm]
\node at (0,0) (A) {$A$};
\node at (1.5,0) (B) {$B$};
\node at (0.75,1) (C) {$C$};
\node at (2.5,0) (D) {$D$};
\path[->] (A) edge (B);
\path[->] (B) edge (D);
\path[<->] (A) edge [bend left] (B);
\path[<->] (A) edge [bend right] (D);
\path[<->] (A) edge (C);
\path[<->] (B) edge (C);
\path[<->] (B) edge [bend left] (D);
\path[<->] (C) edge [bend left] (D);
\end{tikzpicture}
&
\begin{tikzpicture}[inner sep=1mm]
\node at (0,0) (A) {$A$};
\node at (1.5,0) (B) {$B$};
\node at (0.75,1) (C) {$C$};
\node at (2.5,0) (D) {$D$};
\path[->] (A) edge (B);
\path[->] (B) edge (D);
\path[-] (A) edge (C);
\path[-] (B) edge (C);
\path[-] (A) edge [bend right] (D);
\end{tikzpicture}\\
\hline
\end{tabular}
\end{center}\caption{Example where $p(D | \wh{A}, C)$ is identifiable from the aADMG but not from the oADMG.}\label{fig:example3}
\end{figure}
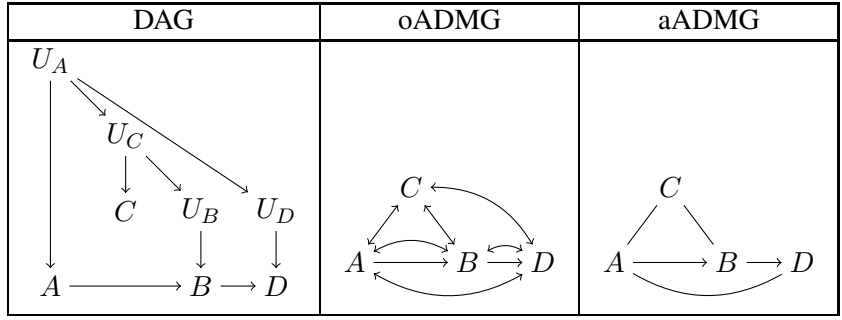

\begin{theorem}
If $p(Y | \wh{X}, W)$ is identifiable from an aADMG $G$ and $W$ contains no descendant of $X$ in $G$, then
\[
p(Y | \wh{X}) = \sum_w p(Y | \wh{X}, w) p(w).
\]
\end{theorem}

\begin{proof}
The only paths between $X$ and $W$ that are not blocked given the empty set are those that reach $X$ through its parents or neighbors. Then, $p(W | \wh{X}) = p(W)$ by rule 3.
\end{proof}

Thanks to the previous theorem, $p(D | \wh{A})$ is identifiable in the example in Figure \ref{fig:example3} ($X=A$, $Y=D$ and $W=C$).

\subsection{Decomposition-Based Causal Effect Identification}

In this section, we present some new graphical criteria for causal effect identification from aADMGs, which add to the back-door and front-door criteria presented in the previous section. First, note that the system of structural equations corresponding to the causal model represented by an aADMG $G$ (recall Equation \ref{eq:equation}) induces a probability distribution $p$ over $V$, namely
\begin{equation}\label{eq:noisy}
p(V) = \sum_u [ \prod_i p(V_i | Pa_G(V_i), u_i) ] p(u)
\end{equation}
where $U_i$ represents the unobserved causes of $V_i$, i.e. the error term $\epsilon_i$ since we are just considering aADMGs, and $U$ represents the set of all $U_i$. Moreover, the distribution induced by the post-interventional system of structural equations can be obtained from the previous equation by simply removing the terms for the variables intervened upon, that is
\begin{equation}\label{eq:postinterventional}
p(V \setminus X | \wh{X}) = \sum_u [ \prod_{V_i \notin X} p(V_i | Pa_G(V_i), u_i) ] p(u) = \sum_{u_{V \setminus X}} [ \prod_{V_i \notin X} p(V_i | Pa_G(V_i), u_i) ] p(u_{V \setminus X}).
\end{equation}

Let $W' \subseteq V$ be an ancestral set, and let $W = V \setminus W'$. We say that two nodes belong to the same component if and only if they are connected by an undirected path in $G_W$, i.e. if and only if they are connected by an undirected path in $G$ that is not blocked by $W'$. Assume that $W$ is partitioned into $k$ components $S_1, \ldots, S_k$. We define the factor
\[
Q(S_j | W') = \sum_{u_{S_j}} [ \prod_{V_i \in S_j} p(V_i | Pa_G(V_i), u_i) ] p(u_{S_j} | U_{W'}).
\]
Note that
\begin{equation}\label{eq:product}
p(W | W') = \prod_j Q(S_j | W').
\end{equation}
To see it, recall that $W'$ is ancestral and, thus, it determines $U_{W'}$. Then, Equation \ref{eq:noisy} implies that
\[
p(W, w') = [ \sum_{u_W} [ \prod_{V_i \in W} p(V_i | Pa_G(V_i), u_i) ] p(u_W | u_{W'}) ] [ \prod_{V_i \in W'} p(v_i | pa_G(V_i), u_i) ] p(u_{W'})
\]
and thus
\begin{equation}\label{eq:conditional}
p(W | w') = \sum_{u_W} [ \prod_{V_i \in W} p(V_i | Pa_G(V_i), u_i) ] p(u_W | u_{W'}) = \prod_j Q(S_j | w')
\end{equation}
because $U_{S_j} \ci_{G'} U_{S_l} | U_{W'}$ for any $j$ and $l$, and where $G'$ denotes the magnification of $G$. Note also that $Q(S_j | W')$ is the distribution of $S_j$ given that the variables in $W'$ are observed and the rest of the variables are intervened upon. To see it, recall again that $W'$ is ancestral and, thus, it determines $U_{W'}$. Therefore, given the observation $W'=w'$ and the intervention $\wh{W \setminus S_j} = \wh{t}$, we have that
\[
Q(S_j | w') = \sum_{u_{S_j}} [ \prod_{V_i \in S_j} p(V_i | Pa_G(V_i), u_i) ] p(u_{S_j} | u_{W'}) \frac{\prod_{V_i \in W'} p(v_i | pa_G(V_i), u_i) ] p(u_{W'})}{\prod_{V_i \in W'} p(v_i | pa_G(V_i), u_i) ] p(u_{W'})}
\]
\[
= \frac{p(S_j, w', \wh{t})}{p(w', \wh{t})} = p(S_j | w', \wh{t})
\]
by Equation \ref{eq:postinterventional}. Moreover, $Q(S_j | W')$ is identifiable as the next theorem shows.

\begin{theorem}
Given an aADMG $G$, $Q(S_j | W')$ is identifiable and is given by
\[
Q(S_j | W') = \prod_{V_i \in S_j} p(V_i | V^{(i-1)}, W')
\]
where $V_1 < \ldots < V_n$ is topological order of $W$ with respect to $G$, and $V^{(i)}=\{V_1, \ldots, V_i\}$.
\end{theorem}

\begin{proof}
We prove the theorem by induction over the number of variables in $W$. Clearly, the theorem holds when $W$ contains a single variable. Assume as induction hypothesis that the theorem holds for up to $n$ variables. When there are $n+1$ variables, these can be divided into components $S_1, \ldots, S_k, S'$ with factors $Q(S_1 | W'), \ldots, Q(S_k | W'), Q(S' | W')$ such that $V_{n+1} \in S'$. By Equation \ref{eq:product}, we have that
\[
p(W | W') = Q(S' | W') \prod_i Q(S_i | W')
\]
which implies that
\[
p(W \setminus V_{n+1} | W') = [ \sum_{v_{n+1}} Q(S' | W') ] \prod_i Q(S_i | W').
\]
Note that $p(W \setminus V_{n+1} | W')$ factorizes according to $(G_W)^{V^{(n)}}$ (i.e. it can be expressed as in Equation \ref{eq:noisy}), and $S_i$ is a component of $(G_W)^{V^{(n)}}$. Therefore,
\[
Q(S_j | W') = \prod_{V_i \in S_j} p(V_i | V^{(i-1)}, W')
\]
by the induction hypothesis and the fact that $V_1 < \ldots < V_n$ is also a topological order of $V^{(n)}$ with respect to $(G_W)^{V^{(n)}}$. Then, $Q(S' | W')$ is also identifiable and is given by
\[
Q(S' | W') = \frac{p(W | W')}{\prod_i Q(S_i | W')} = \frac{\prod_i p(V_i | V^{(i-1)}, W')}{\prod_i Q(S_i | W')} = \prod_{V_i \in S'} p(V_i | V^{(i-1)}, W').
\]
\end{proof}

The following theorem gives a necessary and sufficient graphical criterion for the identification of the causal effect of a single variable on the rest of the variables. 

\begin{theorem}\label{the:children}
Given an aADMG $G$, let $X \in W$ belong to the component $S^X$. Then, $p(W \setminus X | \wh{X}, W')$ is identifiable if and only if there is no undirected path between $X$ and its children in $G_W$. When $p(W \setminus X | \wh{X}, W')$ is identifiable, it is given by
\[
p(W \setminus X | \wh{X}, W') = [ \sum_x Q(S^X | W') ] \prod_i Q(S_i | W').
\]
\end{theorem}

\begin{proof}
To prove the sufficiency part, let $\wh{Q}(S^X | W')$ denote the factor $Q(S^X | W')$ with the term $p(X | Pa_G(X), U_X)$ removed. Note that
\[
p(W \setminus X | \wh{X}, W') = \wh{Q}(S^X | W') \prod_i Q(S_i | W').
\]
As shown before, each $Q(S_i | W')$ is identifiable. Therefore, $p(W \setminus X | \wh{X}, W')$ is identifiable if and only if $\wh{Q}(S^X | W')$ is so. Moreover, since there is no undirected path between $X$ and its children in $G_W$, this implies that no child of $X$ is in $S^X$, which implies that
\[
\wh{Q}(S^X | W') = \sum_x Q(S^X | W').
\]

To prove the necessity part, note that if a causal effect is not identifiable from an aADMG then it is not identifiable if additional edges are added to the aADMG. To see it, note that the additional edges can be made ineffective through the parameters of the corresponding system of structural equations. Therefore, to prove the necessity part of the theorem, it suffices to consider any subgraph of $G_W$ that is of the form of the aADMG in Figure \ref{fig:example5}. The figure also shows the corresponding magnified aADMG and a causal DAG that may have induced the aADMG. We define the following system of structural equations for the DAG:
\begin{align*}
U_X &\sim N(\mu,\sigma)\\
U_{Z_1} &= U_X\\
U_{Z_i} &= U_{Z_{i-1}} \text{ for all $i > 1$}\\
U_A &= \alpha_1 U_{Z_m}\\
X &= U_X\\
Z_i &= U_{Z_i} \text{ for all $i$}\\
A &= \beta_1 X + U_A.
\end{align*} 
Let $p_1(X, Z_1, \ldots, Z_m, A, U_X, U_{Z_1}, \ldots, U_{Z_m}, U_A)$ denote the probability distribution induced by the equations above. We create a second system of structural equations by replacing $\alpha_1$ and $\beta_1$ with $\alpha_2$ and $\beta_2$ such that $\alpha_1 + \beta_1 = \alpha_2 + \beta_2$. Let $p_2(X, Z_1, \ldots, Z_m, A, U_X, U_{Z_1}, \ldots, U_{Z_m}, U_A)$ denote the probability distribution induced by the second system of equations. Finally, note that 
\[
p_i(X, Z_1, \ldots, Z_m, A) = p_i(A | X, Z_1, \ldots, Z_m) p_i(X, Z_1, \ldots, Z_m)
\]
with $i \in \{1,2\}$. Note also that $p_i(X, Z_1, \ldots, Z_m) = 0$ unless $X = Z_1 = \ldots = Z_m$, in which case $p_i(X, Z_1, \ldots, Z_m) = p_i(X)$ with $X \sim N(\mu,\sigma)$. Note also that $p_i(A | X, Z_1, \ldots, Z_m) = 1$ if and only if $A = (\alpha_i + \beta_i) Z_m$, because $X=Z_m$. Then, $p_1(X, Z_1, \ldots, Z_m, A) = p_2(X, Z_1, \ldots, Z_m, A)$. However, 
\begin{align*}
p_i(Z_1, \ldots, Z_m, A | \wh{X}) &= p_i(A | \wh{X}, Z_1, \ldots, Z_m) p_i(Z_1, \ldots, Z_m | \wh{X})\\
&=p_i(A | \wh{X}, Z_1, \ldots, Z_m) p_i(Z_1, \ldots, Z_m)
\end{align*}
by rule 3 in Section \ref{sec:calculus}. Recall that $p_1(Z_1, \ldots, Z_m) = p_2(Z_1, \ldots, Z_m)$. Moreover, note that $p_i(A | \wh{X}, Z_1, \ldots, Z_m) = 1$ if and only if $A = \beta_i \wh{X} + \alpha_i Z_m$. Therefore, $p_1(Z_1, \ldots, Z_m, A | \wh{X}) \neq p_2(Z_1, \ldots, Z_m, A | \wh{X})$. In other words, the causal effect in the theorem cannot be computed uniquely from observed quantities.
\end{proof}

\begin{figure}
\begin{center}
\begin{tabular}{|c|c|c|}
\hline
aADMG & Magnified aADMG & DAG\\
\hline
\begin{tikzpicture}[inner sep=1mm]
\node at (0,0) (X) {$X$};
\node at (4,0) (Y) {$A$};
\node at (1,1) (Z1) {$Z_1$};
\node at (2,1) (d) {$\ldots$};
\node at (3,1) (Zm) {$Z_m$};
\path[->] (X) edge (Y);
\path[-] (Z1) edge (d);
\path[-] (Zm) edge (d);
\path[-] (Zm) edge (Y);
\path[-] (X) edge (Z1);
\end{tikzpicture}
&
\begin{tikzpicture}[inner sep=1mm]
\node at (0,0) (X) {$X$};
\node at (4,0) (Y) {$A$};
\node at (1,1) (Z1) {$Z_1$};
\node at (2,2) (d) {$\ldots$};
\node at (3,1) (Zm) {$Z_m$};
\node at (0,2) (UX) {$U_X$};
\node at (4,2) (UY) {$U_A$};
\node at (1,2) (UZ1) {$U_{Z_1}$};
\node at (3,2) (UZm) {$U_{Z_m}$};
\path[->] (X) edge (Y);
\path[->] (UX) edge (X);
\path[->] (UY) edge (Y);
\path[->] (UZ1) edge (Z1);
\path[->] (UZm) edge (Zm);
\path[-] (UZ1) edge (d);
\path[-] (UZm) edge (d);
\path[-] (UZm) edge (UY);
\path[-] (UX) edge (UZ1);
\end{tikzpicture}
&
\begin{tikzpicture}[inner sep=1mm]
\node at (0,0) (X) {$X$};
\node at (4,0) (Y) {$A$};
\node at (1,1) (Z1) {$Z_1$};
\node at (2,2) (d) {$\ldots$};
\node at (3,1) (Zm) {$Z_m$};
\node at (0,2) (UX) {$U_X$};
\node at (4,2) (UY) {$U_A$};
\node at (1,2) (UZ1) {$U_{Z_1}$};
\node at (3,2) (UZm) {$U_{Z_m}$};
\path[->] (X) edge (Y);
\path[->] (UX) edge (X);
\path[->] (UY) edge (Y);
\path[->] (UZ1) edge (Z1);
\path[->] (UZm) edge (Zm);
\path[->] (UZ1) edge (d);
\path[<-] (UZm) edge (d);
\path[->] (UZm) edge (UY);
\path[->] (UX) edge (UZ1);
\end{tikzpicture}\\
\hline
\end{tabular}
\end{center}\caption{Example in the proof of Theorem \ref{the:children}.}\label{fig:example5}
\end{figure}
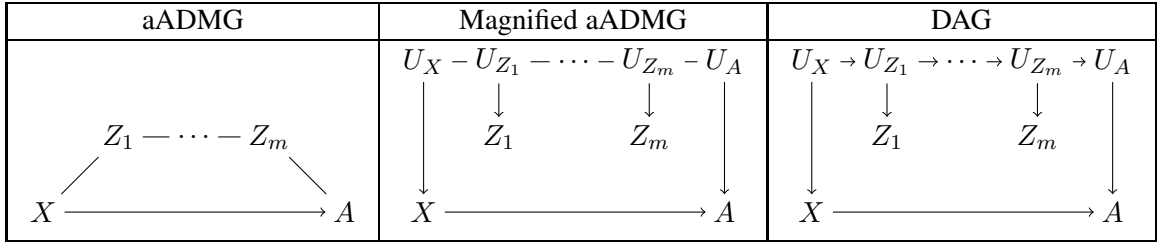

For instance, $p(B, C, D | \wh{A})$ is not identifiable by the previous theorem from the aADMG in Figure \ref{fig:example3} ($X=A$, $W=\{ A, B, C, D \}$ and $W'= \emptyset$). However, $p(B, D | \wh{A}, C)$ is identifiable ($X=A$, $W=\{ A, B, D \}$ and $W'= C$).

The following theorem strengthens the sufficient condition for unidentifiability in the previous theorem.

\begin{theorem}\label{the:children2}
Given an aADMG $G$, let $X \in W'' \subseteq W$. Then, $p(W'' \setminus X | \wh{X}, W')$ is not identifiable if there is an undirected path between $X$ and its children in $G_{W''}$.
\end{theorem}

\begin{proof}
Note that $p(W'' \setminus X | \wh{X}, W')$ is not identifiable from the aADMG $G_{W' \cup W''}$ by Theorem \ref{the:children}. Then, it is not identifiable from $G$ either because the additional edges can be made ineffective through the parameters of the corresponding system of structural equations.
\end{proof}
 
Clearly, whenever $p(W \setminus X | \wh{X}, W')$ is identifiable by Theorem \ref{the:children}, so is $p(Y \setminus X | \wh{X}, W')$ with $Y \subseteq W$ by marginalization. However, as the following theorem shows, there are cases where the latter is identifiable despite the fact that the former is not. Let $W''= An_G(Y) \setminus W'$.

\begin{theorem}\label{the:marginalchildren}
Given an aADMG $G$, if there is no undirected path between $X \in W''$ and its children in $(G_W)^{W''}$, then $p(W'' \setminus X | W', \wh{X})$ is identifiable and is given by
\[
p(W'' \setminus X | \wh{X}, W') = [ \sum_x Q(T^X | W') ] \prod_i Q(T_i | W')
\]
where $T_1, \ldots, T_k, T^X$ are the components in which $W''$ is partitioned in $(G_W)^{W''}$, and $T^X$ is the component that contains $X$.
\end{theorem}

\begin{proof}
By summing over $W \setminus W''$ on both sides of Equation \ref{eq:conditional}, we have that
\[
p(W'' | W') = \sum_{u_{W}} [ \prod_{V_i \in W''} p(V_i | Pa_G(V_i), u_i) ] p(u_{W} | u_{W'})
\]
\[
= \sum_{u_{W''}} [ \prod_{V_i \in W''} p(V_i | Pa_G(V_i), u_i) ] p(u_{W''} | u_{W'}).
\]
Note that the previous expression is of the same form as Equation \ref{eq:noisy}, when $G$ is replaced with $(G_W)^{W''}$. Then, we can repeat the reasoning in the proof of Theorem \ref{the:children}.
\end{proof}

\begin{theorem}\label{the:marginalchildren2}
Given an aADMG $G$, if there is no undirected path between $X \in W''$ and its children in $(G_W)^{W''}$, then $p(W'' \setminus X, W' | \wh{X})$ and $p(W'' \setminus X | \wh{X})$ are identifiable and are given by
\[
p(W'' \setminus X, W' | \wh{X}) = p(W') [ \sum_x Q(T^X | W') ] \prod_i Q(T_i | W')
\]
and
\[
p(W'' \setminus X | \wh{X}) = \sum_{w'} p(w') [ \sum_x Q(T^X | w') ] \prod_i Q(T_i | w')
\]
where $T_1, \ldots, T_k, T^X$ are as in Theorem \ref{the:marginalchildren}.
\end{theorem}

\begin{proof}
The theorem follows from Theorem \ref{the:marginalchildren} by noting that $p(W' | \wh{X}) = p(W')$, because $W'$ is an ancestral set.
\end{proof}

For instance, $p(B, D | \wh{A}, C)$ is not identifiable by Theorem \ref{the:marginalchildren} from the aADMG in Figure \ref{fig:example4} ($X=A$, $W''=\{ A, B, D \}$ and $W'=C$). However, $p(B | \wh{A}, C)$ is identifiable ($X=A$, $W''=\{ A, B \}$ and $W'=C$). Moreover, $p(B, C | \wh{A})$ and $p(B | \wh{A})$ are identifiable by Theorem \ref{the:marginalchildren2} ($X=A$, $W''=\{ A, B \}$ and $W'=C$). Note that the latter two effects are not identifiable by Theorem \ref{the:marginalchildren}. Note also that neither $p(B | \wh{A}, C)$, $p(B, C | \wh{A})$ nor $p(B | \wh{A})$ are identifiable from the oADMG in the figure \citep[p. 94]{Pearl2009}.

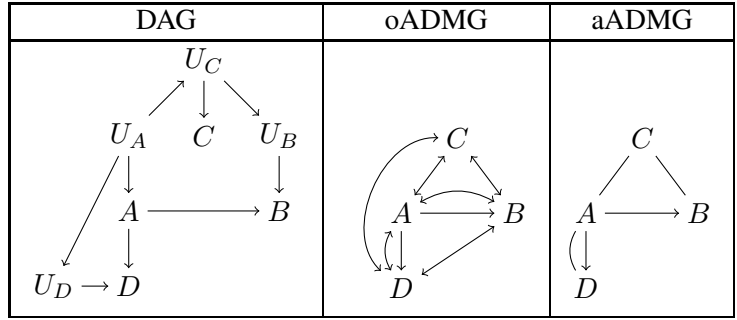
\begin{figure}
\begin{center}
\begin{tabular}{|c|c|c|}
\hline
DAG&oADMG&aADMG\\
\hline
\begin{tikzpicture}[inner sep=1mm]
\node at (0,0) (A) {$A$};
\node at (2,0) (B) {$B$};
\node at (1,1) (C) {$C$};
\node at (0,1) (UA) {$U_A$};
\node at (2,1) (UB) {$U_B$};
\node at (1,2) (UC) {$U_C$};
\node at (0,-1) (D) {$D$};
\node at (-1,-1) (UD) {$U_D$};
\path[->] (A) edge (B);
\path[->] (UA) edge (A);
\path[->] (UB) edge (B);
\path[->] (UC) edge (C);
\path[->] (UA) edge (UC);
\path[->] (UC) edge (UB);
\path[->] (UD) edge (D);
\path[->] (A) edge (D);
\path[->] (UA) edge (UD);
\end{tikzpicture}
&
\begin{tikzpicture}[inner sep=1mm]
\node at (0,0) (A) {$A$};
\node at (1.5,0) (B) {$B$};
\node at (0.75,1) (C) {$C$};
\node at (0,-1) (D) {$D$};
\path[->] (A) edge (B);
\path[<->] (A) edge [bend left] (B);
\path[<->] (A) edge (C);
\path[<->] (B) edge (C);
\path[->] (A) edge (D);
\path[<->] (A) edge [bend right] (D);
\path[<->] (B) edge (D);
\path[<->] (C) edge [bend right=70] (D);
\end{tikzpicture}
&
\begin{tikzpicture}[inner sep=1mm]
\node at (0,0) (A) {$A$};
\node at (1.5,0) (B) {$B$};
\node at (0.75,1) (C) {$C$};
\node at (0,-1) (D) {$D$};
\path[->] (A) edge (B);
\path[-] (A) edge (C);
\path[-] (B) edge (C);
\path[->] (A) edge (D);
\path[-] (A) edge [bend right] (D);
\end{tikzpicture}\\
\hline
\end{tabular}
\end{center}\caption{Example where $p(B | \wh{A}, C)$ and $p(B | \wh{A})$ are identifiable from the aADMG but not from the oADMG.}\label{fig:example4}
\end{figure}

Note that $X$ is a singleton in Theorem \ref{the:marginalchildren}. The algorithm in Table \ref{tab:multiple} shows how Theorem \ref{the:marginalchildren} can be used for causal effect identification when $X$ is a subset of $V$. For each ordering $\sigma$ of the variables in $X$, the algorithm tries to identify the causal effect of $X_{\sigma(1)}$ from the available information, i.e. $p(W'' | W')$ and $(G_W)^{W''}$. If the identification succeeds, then the algorithm tries to identify the causal effect of $X_{\sigma(2)}$ from the available information, i.e. the post-interventional distribution identified in the previous iteration and the corresponding aADMG $(G')_{\wh{X_{\sigma(1)}}}$ where $G'=(G_W)^{W''}$. The process continues until the causal effects of all variables in $X$ are identified in which case the last expression constitutes the answer to the original query, or some effect is not identifiable in which case the algorithm tries a new ordering $\sigma$. If all orderings are tried without success, then the algorithm declares the causal effect unidentifiable. Whether the effect is truly unidentifiable is still an open problem.

We illustrate the use of the algorithm in Table \ref{tab:multiple} with an example. In particular, consider again the example in Figure \ref{fig:example1}. Then, $p(B | \wh{A}, \wh{C})$ is identifiable from the aADMG $G$ with the ordering $\sigma=(C, A)$: First, $p(A, B | \wh{C})$ is identifiable from $G$ by Theorem \ref{the:marginalchildren} and is given by a function of $p(A, B, C)$ and, then, $p(B | \wh{C}, \wh{A})$ is identifiable from $G_{\wh{C}} = \{ A \ra B, C \}$ and is given by a function of $p(A, B | \wh{C})$. Note that the effect is not identifiable with the ordering $\sigma=(A, C)$.

Note that Theorem \ref{the:marginalchildren2} can also be generalized to the case where $X$ is a subset of $V$: If $p(W'' \setminus X | \wh{X}, W')$ is identifiable by the algorithm in Table \ref{tab:multiple}, then $p(W'' \setminus X, W' | \wh{X})$ and $p(W'' \setminus X | \wh{X})$ are also identifiable and are given by
\[
p(W'' \setminus X, W' | \wh{X}) = p(W') p(W'' \setminus X | \wh{X}, W')
\]
and
\[
p(W'' \setminus X | \wh{X}) = \sum_{w'} p(w') p(W'' \setminus X | \wh{X}, w')
\]
because $p(W' | \wh{X}) = p(W')$ since $W'$ is an ancestral set. Finally, we have the following result.

\begin{table}
\caption{Causal effect identification for multiple interventions.}\label{tab:multiple}
\begin{center}
\begin{tabular}{|rl|}
\hline
1 & $V^1=W''$\\
2 & $p^1(V^1)=p(W'' | W')$\\
3 & $G^1=(G_W)^{W''}$\\
4 & for each ordering $\sigma$ of the variables in $X$ do\\
5 & \hspace{0.3cm} for $i=1, \ldots, |X|$ do\\
6 & \hspace{0.6cm} if $p^i(V^i \setminus X_{\sigma(i)} | \wh{X_{\sigma(i)}})$ is identifiable from $G^i$ by Theorem \ref{the:marginalchildren} then\\
7 & \hspace{0.9cm} $V^{i+1}=V^i \setminus X_{\sigma(i)}$\\
8 & \hspace{0.9cm} $p^{i+1}(V^{i+1})=p^i(V^i \setminus X_{\sigma(i)} | \wh{X_{\sigma(i)}})$\\
9 & \hspace{0.9cm} $G^{i+1}=(G^i)_{\wh{X_{\sigma(i)}}}$\\
10 & \hspace{0.6cm} else go to line 4\\
11 & \hspace{0.3cm} return IDENTIFIABLE\\
12 & return UNIDENTIFIABLE\\
\hline
\end{tabular}
\end{center}
\end{table}

\begin{theorem}
Given an aADMG $G$, $p(Y | \wh{X}, W')$ is identifiable if every undirected path between $X$ and those children of $X$ that are ancestor of $Y$ includes some node that is in $W'$ or that is neither a descendant of $X$ nor of $Y$.
\end{theorem}

\begin{proof}
Let $T$ denote the nodes that are neither descendants of $X$ nor of $Y$. Note that $W' \cup T$ is an ancestral set. Then, $p((W'' \setminus T ) \setminus X, W', T | \wh{X})$ is identifiable by Theorem \ref{the:marginalchildren2} with $W' \cup T$ instead of $W'$, $W'' \setminus T$ instead of $W''$, and $W \setminus T$ instead of $W$. Then, the theorem follows by conditioning and marginalization.
\end{proof}


\section{Gated Models for Causal Effect Identification}\label{sec:csi}

The gambling example in Section~\ref{sec:introduction} showed how the context may decide which of the oADMG or the aADMG is more appropriate in terms of being able to identify the causal effect of interest. We used a gated model, i.e. a graphical model which combines multiple graphical models using gates, to explicitly state the criteria used to decide which model is more appropriate. In this section we shall generalise this idea, using gated models to accommodate so called context specific independences (CSIs), in such a way that we can exploit these independences to identify more causal effects than would be possible using a single model. The gated model can also be used to model certain causal phenomena that may occur due to CSIs, such as unstable and non-deterministic effects of interventions. It should be noted that gated models are just a formalism to represent contexts explicitly, so that the {\it do}-calculi introduced in the previous sections and in other articles can be fully deployed. So, gated models can build on any existing family of graphical models, e.g. oADMGs, aADMGs, ADMGs, etc.

\begin{figure}[t]
	\centering
	\begin{minipage}[t]{0.49\textwidth}
		\begin{subfigure}[t]{\textwidth}
		\centering
		\resizebox{!}{2.5cm} {

\begin{tikzpicture}

	\begin{scope}
		\node[var] (X) {X};
		\node[var] (Z)[right=of X] {Z};
		\node[var] (Y)[below= of Z] {Y};
		\path (X) edge [connect] (Y);
		\path (X) edge [connect] (Z);
		\path (Z) edge [connect] (Y);
		
		\node[box, fit= (X) (Y) (Z), label=above:$$] (O1) {};
	\end{scope}
	
\end{tikzpicture}
		}
		\caption{}
		\label{fig:introduction-example1-bn}
		\end{subfigure}
	\end{minipage}
	\hfill
	\begin{minipage}[t]{0.49\textwidth}
		\begin{subfigure}[t]{\textwidth}
		\centering
		\resizebox{!}{2.5cm} {

\begin{tikzpicture}

	\begin{scope}
		\node[var] (X) {X};
		\node[var] (Z)[right=of X] {Z};
		\node[var] (Y)[below= of Z] {Y};
		\path (X) edge [connect] (Y);
		\path (X) edge [connect] (Z);
		\path (Z) edge [connect] (Y);
		\node[box, fit= (X) (Y) (Z), label=above:$$] (O1) {};
	\end{scope}
	
	\begin{scope}[xshift=4.0cm]
		\node[var] (X) {X};
		\node[var] (Z)[right=of X] {Z};
		\node[var] (Y)[below= of Z] {Y};
		\path (X) edge [connect] (Y);
		\path (Z) edge [connect] (Y);
		\node[box, fit= (X) (Y) (Z), label=left:$$] (O2) {};
	\end{scope}
			
	\node[gate,right = of O1,xshift=-0.5cm,yshift=0.7cm] (G1) {$X  > 0$};
	\path (O1) edge [connect] (G1);
	\path (G1) edge [connect] (O2);
	
	\node[gate,left=of O2,xshift=0.5cm,yshift=-0.7cm] (G2) {$X \leq 0$};
	\path (O2) edge [connect] (G2);
	\path (G2) edge [connect] (O1);

\end{tikzpicture}
		}
		\caption{}
		\label{fig:introduction-example1-gated}
		\end{subfigure}
	\end{minipage}
	
	\vspace{0.5cm}
	
	\centering
	\begin{minipage}[t]{0.47\textwidth}
		\begin{subfigure}[t]{\textwidth}
		\centering
		\resizebox{!}{2.5cm} {

\begin{tikzpicture}

	\begin{scope}
		\node[var] (X) {X};
		\node[var] (Z)[right=of X] {Z};
		\node[var] (U1)[right=of Z] {$U_1$};
		\node[var] (U2)[right=of U1] {$U_2$};		
		\node[var] (Y)[below= of U1] {Y};
		\path (X) edge [connect] (Y);
		\path (X) edge [connect] (Z);
		\path (Z) edge [connect] (U1);
		\path (U1) edge [connect] (Y);
		\path (U2) edge [connect] (U1);
		\node[box, fit= (X) (Y) (Z) (U1) (U2), label=above:$$] (O1) {};
	\end{scope}
	
\end{tikzpicture}
		}
		\caption{}
		\label{fig:introduction-example2-bn}
		\end{subfigure}
	\end{minipage}
	\hfill
	\begin{minipage}[t]{0.52\textwidth}
		\begin{subfigure}[t]{\textwidth}
		\centering
		\resizebox{!}{3.3cm} {

\begin{tikzpicture}

	\begin{scope}
		\node[var] (X) {X};
		\node[var] (Z)[right=of X] {Z};
		\node[var] (Y)[below= of Z] {Y};
		\path (X) edge [connect] (Y);
		\path (X) edge [connect] (Z);
		\path (Z) edge [connect] (Y);
		\node[box, fit= (X) (Y) (Z), label=left:$R_1$] (O1) {};
	\end{scope}
	
	\begin{scope}[xshift=4.0cm]
		\node[var] (X) {X};
		\node[var] (Z)[right=of X] {Z};
		\node[var] (Y)[below= of Z] {Y};
		\path (X) edge [connect] (Y);
		\path (X) edge [connect] (Z);
		\node[box, fit= (X) (Y) (Z), label=right:$R_2$] (O2) {};
	\end{scope}
			
	\node[gate,right = of O1,xshift=-1.5cm,yshift=1.3cm] (G1) {\scriptsize $L(D | R_2) / L(D | R_1) > \theta$};
	\path (O1) edge [connect] (G1);
	\path (G1) edge [connect] (O2);
	
	\node[gate,left=of O2,xshift=1.5cm,yshift=-1.3cm] (G2) {\scriptsize $L(D | R_1) / L(D | R_2) > \theta$};
	\path (O2) edge [connect] (G2);
	\path (G2) edge [connect] (O1);

\end{tikzpicture}
		}
		\caption{}
		\label{fig:introduction-example2-gated}
		\end{subfigure}
	\end{minipage}
	\caption{In (a), a single DAG that does not convey that $X \rightarrow Z$ can be removed if $X > 0$, a gated model representing this extra knowledge is given in (b). When the CSI is dependent on unobserved variables, $U_1$ and $U_2$ in (c), we cannot discern context based on variables taking specific values. The gated model in (d) uses threshold gates to decide which DAG is appropriate.}
\end{figure}
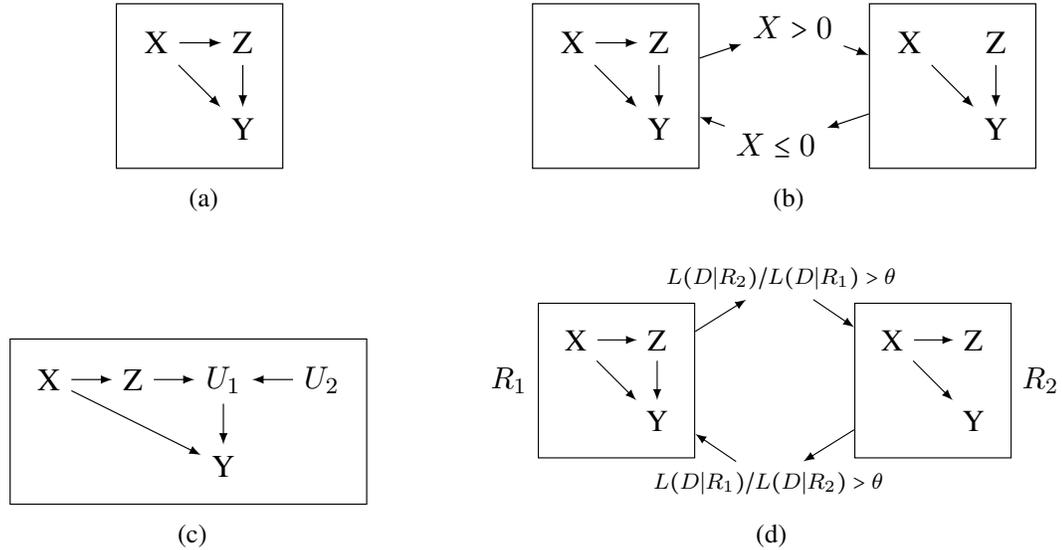

\subsection{Context Specific Independence}

In a graphical model, an edge between two nodes should be present if there is a potential direct association between the two variables that the nodes represent, even if this association is only present for a certain range of the values the variables may take. Studying the conditional probability distributions that are associated with the model may however reveal that some edges can be removed when some variables take on certain values. Exploiting such CSIs within the domain of DAGs has been done since at least the mid-nineties \citep{boutilier-1996}, where representing discrete distributions which contained such CSIs in form of trees allowed for more efficient representation. CSIs were studied a few years earlier in the domain of influence diagrams \citep{smith-1993}, and later within the domain of event trees \citep{smith-2008}. \citet{poole-2003} give several examples of cases where CSIs may materialize, and they extend the variable elimination algorithm with the ability to utilize CSIs within the DAG, resulting in more efficient inference. Enforcing CSIs has also been used to facilitate more efficient inference and parameter estimation, for instance in the form of noisy OR-models \citep{pearl-1988,zagorecki-2004}. To further illustrate CSIs, consider the DAG in \figurename~\ref{fig:introduction-example1-bn}, and the additional knowledge that the edge $X \rightarrow Z$ is unnecessary when $X > 0$. This CSI is obviously not revealed by the DAG in the figure. An alternative way of representing this DAG is to use a gated model, an example of which is depicted in \figurename~\ref{fig:introduction-example1-gated}. In \figurename~\ref{fig:introduction-example1-gated} we have two DAGs connected with gates, when $X>0$ we shall use the DAG to the right, where there is no edge between $X$ and $Z$, and when $X \leq 0$ then we shall use the DAG to the left (where the edge is present). The gated model representation allows us to not only reduce the number of edges in some circumstances  (which in turn may imply less computation), but also graphically state the CSIs that exist among the variables. 

In some cases we may have CSIs that are due to unobserved variables. Consider the DAG in \figurename~\ref{fig:introduction-example2-bn}, where we assume that $U_1$ and $U_2$ are unobserved. If the edge between $Z$ and $U_1$ disappears when $U_2 > 0$, then in a dataset over $X$, $Y$ and $Z$ there will be data points for which the edge $Z \rightarrow Y$ has to be present, and data points for which it is not required, however this may not be evident from the conditional distributions estimated from the dataset. We can therefore not create gates which specify for which values of any of the observed variables we can remove $Z \rightarrow Y$. In such cases we use a gate which judges which DAG is more appropriate given some data $D$. In our previous work, e.g. \citep{bendtsen-regime-baseball, bendtsen-online}, appropriate has been equated with likelihood, which we shall denote with $L$. \figurename~\ref{fig:introduction-example2-gated} shows an example, where we have labeled the DAGs $R_1$ and $R_2$. The model makes it clear that it is necessary to continually assess whether or not the other DAG is more appropriate (above some threshold $\theta$).

\begin{figure}[t]
	\centering
	\begin{minipage}[t]{0.39\textwidth}
		\begin{subfigure}[t]{\textwidth}
		\centering
		\resizebox{!}{2.1cm} {

\begin{tikzpicture}

	\node[var] (T) {T};
	\node[var, right=of T] (C) {C};
	\node[var, right=of C] (A) {A};
	\node[var, below=of C] (W) {W};
	
	\path (T.north) edge [-, bend left= 50] (C.north);
	\path (T) edge [connect] (C);
	\path (C) edge [connect] (A);
	\path (W) edge [connect] (C);
	
	\path [draw, -latex] (T) to[bend right= 45] (W.south) to[bend right = 45] (A);
	
	\node[box, inner sep=15pt, fit= (T) (C) (W) (A), label=left:$$] (O1) {};
	
\end{tikzpicture}
		}
		\caption{}
		\label{fig:admg-csi-w}
		\end{subfigure}
	\end{minipage}
	\hfill
	\begin{minipage}[t]{0.60\textwidth}
		\begin{subfigure}[t]{\textwidth}
		\centering
		\resizebox{!}{2.1cm} {

\begin{tikzpicture}

	\begin{scope}
		\node[var] (T) {T};
		\node[var, right=of T] (C) {C};
		\node[var, right=of C] (A) {A};
		\node[var, below=of C] (W) {W};
	
		\path (T.north) edge [-, bend left= 50] (C.north);
		\path (T) edge [connect] (C);
		\path (C) edge [connect] (A);
		\path (W) edge [connect] (C);
	
		\path [draw, -latex] (T) to[bend right= 45] (W.south) to[bend right = 45] (A);

		\node[box, inner sep=15pt, fit= (T) (C) (W) (A), label=left:$$] (O1) {};
	\end{scope}
	
	\begin{scope}[xshift=6.5cm]
		\node[var] (T) {T};
		\node[var, right=of T] (C) {C};
		\node[var, right=of C] (A) {A};
		\node[var, below=of C] (W) {W};
	
		\path (T.north) edge [-, bend left= 50] (C.north);

		\path (C) edge [connect] (A);
		\path (W) edge [connect] (C);
	
		\path [draw, -latex] (T) to[bend right= 45] (W.south) to[bend right = 45] (A);

		\node[box,  inner sep=15pt, fit= (T) (C) (W) (A), label=left:$$] (O2) {};	
	\end{scope}
			
	\node[gate,right = of O1,xshift=-0.5cm,yshift=0.7cm] (G1) {$\wh{W}  > 100$};
	\path (O1) edge [connect] (G1);
	\path (G1) edge [connect] (O2);
	
	\node[gate,left=of O2,xshift=0.5cm,yshift=-0.7cm] (G2) {$\wh{W}  \leq  100$};
	\path (O2) edge [connect] (G2);
	\path (G2) edge [connect] (O1);	
	
\end{tikzpicture}
		}
		\caption{}
		\label{fig:admg-csi-w-gbn}
		\end{subfigure}
	\end{minipage}
	\caption{In (a), $p(A, C | \wh{T})$ is not identifiable. In (b), in the context where $\wh{W} > 100$, $p(A, C | \wh{T})$ is identifiable.}
\end{figure}
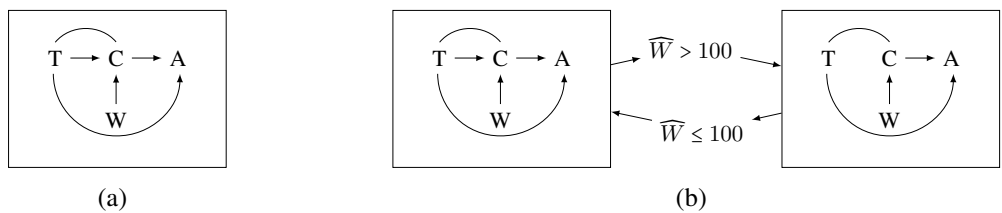

While being able to remove an edge in a graphical model may result in faster computation and easier estimation of parameters, removing an edge in a causal model may lead to causal effects being identifiable that previously were not. Consider the aADMG in \figurename~\ref{fig:admg-csi-w}. Due to Theorem \ref{the:children2}, we cannot identify $p(A, C | \wh{T})$. However, if we could identify a CSI, for instance when $\wh{W} > 100$ the edge $T \rightarrow C$ can be removed, then $p(A, C | \wh{T})$ would be identifiable in this specific context. The aADMG in \figurename~\ref{fig:admg-csi-w} however lacks the ability to tell us that $p(A, C | \wh{T})$ is identifiable in the context of $\wh{W} > 100$. A gated model such as the one in \figurename~\ref{fig:admg-csi-w-gbn} solves this problem by making the contexts and the context specific causal models explicit.

\subsection{Modeling Causal Scenarios Using Gated Models}
\label{sec:examples}

Gated models are interesting for causal reasoning not only because they make explicit the contexts where causal effects are identifiable, but also because they allow modeling some causal phenomena that occur due to CSIs. Without having the possibility to explicitly represent CSIs, these causal phenomena would pass unmodeled. We describe three such causal phenomena below. For simplicity, we use DAGs in some of the examples below. However, as mentioned above, gated models can build on any other family of graphical models.

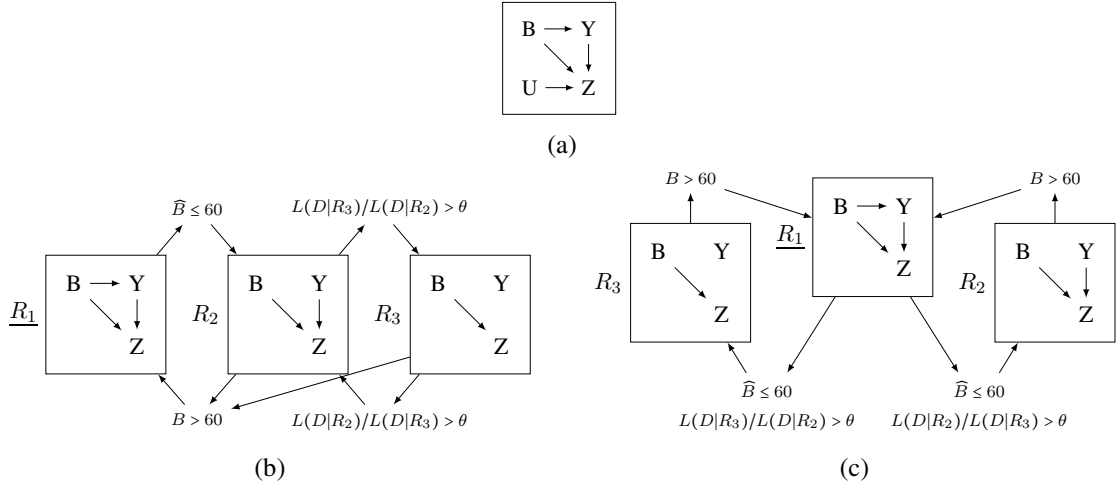
\begin{figure}[t]
	\centering
	\begin{minipage}[t]{\textwidth}
		\begin{subfigure}[t]{\textwidth}
		\centering
		\resizebox{!}{1.5cm} {

\begin{tikzpicture}
	\begin{scope}
		\node[var] (B) {B};
		\node[var] (Y)[right=of B] {Y};
		\node[var] (Z)[below= of Y] {Z};
		\node[var] (U)[left= of Z] {U};
		\path (B) edge [connect] (Y);
		\path (B) edge [connect] (Z);
		\path (Y) edge [connect] (Z);
		\path (U) edge [connect] (Z);
		\node[box, fit= (B) (Y) (Z) (U)] (O1) {};
	\end{scope}
\end{tikzpicture}
		}
		\caption{}
		\label{fig:examples-unstable-truth}
		\end{subfigure}
	\end{minipage}
	\begin{minipage}[t]{0.49\textwidth}
		\begin{subfigure}[t]{\textwidth}
		\centering
		\resizebox{\linewidth}{!} {

\resizebox{1.0\linewidth}{!} {
\centering
		
\begin{tikzpicture}

	\begin{scope}
		\node[var] (B) {B};
		\node[var] (Y)[right=of B] {Y};
		\node[var] (Z)[below= of Y] {Z};
		\path (B) edge [connect] (Y);
		\path (B) edge [connect] (Z);
		\path (Y) edge [connect] (Z);
		\node[box, fit= (B) (Y) (Z), label=left:$\underline{R_1}$] (O1) {};
		
		\node[gate,above=of O1,xshift=1.5cm,yshift=-0.5cm] (G1) {\scriptsize $\wh{B} \leq 60$};
		\path (O1) edge [connect] (G1);
	\end{scope}
	
	\begin{scope}[xshift=3.0cm]
		\node[var] (B) {B};
		\node[var] (Y)[right=of B] {Y};
		\node[var] (Z)[below= of Y] {Z};
		\path (B) edge [connect] (Z);
		\path (Y) edge [connect] (Z);
		\node[box, fit= (B) (Y) (Z), label=left:$R_2$] (I1) {};
		
		\node[gate,above=of I1,xshift=1.5cm,yshift=-0.5cm] (G2) {\scriptsize $L(D|R_3) / L(D|R_2) > \theta$};
		\node[gate,below=of I1,xshift=-1.5cm,yshift=0.5cm] (G4) {\scriptsize $B > 60$};
		\path (I1) edge [connect] (G2);
		\path (I1) edge [connect] (G4);
	\end{scope}
	
	\begin{scope}[xshift=6.0cm]
		\node[var] (B) {B};
		\node[var] (Y)[right=of B] {Y};
		\node[var] (Z)[below= of Y] {Z};
		\path (B) edge [connect] (Z);
		\node[box, fit= (B) (Y) (Z), label=left:$R_3$] (I2) {};
		
		\node[gate,below=of I2,xshift=-1.5cm,yshift=0.5cm] (G3) {\scriptsize $L(D|R_2) / L(D|R_3) > \theta$};
		\path (I2) edge [connect] (G3);
		\path (I2.west) edge [yshift=-0.7cm,connect] (G4);
	\end{scope}
		
	\path (G1) edge [connect] (I1);
	\path (G2) edge [connect] (I2);
	\path (G3) edge [connect] (I1);
	\path (G4) edge [connect] (O1);

\end{tikzpicture}

}
		}
		\caption{}
		\label{fig:examples-unstable-gated}
		\end{subfigure}
	\end{minipage}
	\hfill
	\begin{minipage}[t]{0.49\textwidth}
		\begin{subfigure}[t]{\textwidth}
		\centering
		\resizebox{\linewidth}{!} {

\resizebox{1.0\linewidth}{!} {
\centering
		
\begin{tikzpicture}

	\begin{scope}
		\node[var] (B) {B};
		\node[var] (Y)[right=of B] {Y};
		\node[var] (Z)[below= of Y] {Z};
		\path (B) edge [connect] (Y);
		\path (B) edge [connect] (Z);
		\path (Y) edge [connect] (Z);
		\node[box, fit= (B) (Y) (Z), label=left:$\underline{R_1}$] (O1) {};
		
		\node[gate,below=of O1,xshift=1.75cm,yshift=-0.25cm] (G1) {\scriptsize $\wh{B} \leq 60$ \\ \scriptsize $L(D|R_2) / L(D|R_3) > \theta$};
		\node[gate,below=of O1,xshift=-1.75cm,yshift=-0.25cm] (G2) {\scriptsize $\wh{B} \leq 60$ \\ \scriptsize $L(D|R_3) / L(D|R_2) > \theta$};
		\path (O1) edge [connect] (G1);
		\path (O1) edge [connect] (G2);
	\end{scope}
	
	\begin{scope}[xshift=3.0cm,yshift=-0.75cm]
		\node[var] (B) {B};
		\node[var] (Y)[right=of B] {Y};
		\node[var] (Z)[below= of Y] {Z};
		\path (B) edge [connect] (Z);
		\path (Y) edge [connect] (Z);
		\node[box, fit= (B) (Y) (Z), label=left:$R_2$] (I1) {};
		
		\node[gate,above=of I1,yshift=-0.5cm] (G3) {\scriptsize $B > 60$};
		\path (I1) edge [connect] (G3);
	\end{scope}
	
	\begin{scope}[xshift=-3.0cm,yshift=-0.75cm]
		\node[var] (B) {B};
		\node[var] (Y)[right=of B] {Y};
		\node[var] (Z)[below= of Y] {Z};
		\path (B) edge [connect] (Z);
		\node[box, fit= (B) (Y) (Z), label=left:$R_3$] (I2) {};
		
		\node[gate,above=of I2,yshift=-0.5cm] (G4) {\scriptsize $B > 60$};
		\path (I2) edge [connect] (G4);
	\end{scope}
		
	\path (G1) edge [connect] (I1);
	\path (G2) edge [connect] (I2);
	\path (G3) edge [connect] (O1);
	\path (G4) edge [connect] (O1);

\end{tikzpicture}

}
		}
		\caption{}
		\label{fig:examples-nondeterministic-gated}
		\end{subfigure}
	\end{minipage}
	\caption{In (a), a single causal model cannot capture the extra knowledge regarding CSIs. In (b), intervening to lower the blood pressure implies moving to a physiological stable state, but may lead to transitions back and forth with a crisis state. In (c), the outcome of the intervention is unknown, it may either lead to the stable or the crisis state.}
\end{figure}

\subsubsection{Unstable Effect and Non-Deterministic Outcome of Interventions}

Assume that our domain of interest consists of the following variables: Blood pressure ($B$), dosage of medication ($Y$), dizziness ($Z$), and anxiety ($U$). Moreover, $B$, $Y$ and $Z$ are observed but $U$ is unobserved. Assume that performing an intervention to lower the blood pressure puts the human body in a context where it goes back and forth between a physiological stable state and an unstable crisis state. After performing the intervention we know that $\wh{B} \leq 60$ and the body enters the physiological stable state, but may then go back and forth between the stable and crisis state. We want to model these states or contexts separately due to the existence of the CSIs below.

In \figurename~\ref{fig:examples-unstable-truth}, a single causal model for this scenario is depicted. This model tells us very little about the scenario just narrated. Knowing that the edge $B \rightarrow Y$ is lost in the context of $\wh{B} \leq 60$ (i.e. the dosage of medication is not a function of the blood pressure when the pressure is low), and $Y \rightarrow Z$ is lost when $\wh{B} \leq 60$ and $U > 0$ (i.e. dizziness is not a function of the dosage of medication when the blood pressure is low and the patient experiences an anxiety attack) does not help in making the model more informative. However, using a gated model we can incorporate this extra knowledge. The resulting gated model is depicted in \figurename~\ref{fig:examples-unstable-gated}. As we can see, once $\wh{B} \leq 60$, we change from $R_1$ to $R_2$, the physiological stable state under low blood pressure. Since $U$ is unobserved, we cannot use a context gate to describe when we should switch between $R_2$ and $R_3$ (the crisis state), thus we need to use a threshold gate that can identify which of the two contexts is most appropriate. The gated model also makes it clear that increasing the blood pressure while either in the crisis state or the stable state will take the body back to the initial context. 

To demonstrate non-deterministic outcomes, we shall change the previous scenario slightly. This time, assume that we do not know if we will enter the stable state or the crisis state after the intervention to lower the blood pressure. Furthermore, the stable and crisis states cannot transition into each other, i.e. once the blood pressure turns low, the body enters a stable or crisis state and does not change between them. The single causal model stays the same (\figurename~\ref{fig:examples-unstable-truth}), and again tells us very little about the scenario. Using a gated model we can model this causal scenario, but this time using a combination of threshold and context gates, as depicted in \figurename~\ref{fig:examples-nondeterministic-gated}. The gated model tells us that when the intervention lowers the blood pressure, we must assess which of the states is most appropriate, and it is the data $D$ that determines the appropriate causal model (recall that this is due to the CSIs that are created due to the unobserved variable $U$).

\subsubsection{Mechanism Dependent Outcome of Interventions}
\label{sec:examples-mechanism-dependent}

Consider the monthly physical activity ($W$) measured in hours, and the monthly alcohol consumption ($C$) of a person in a population. Assume that there exists a leaflet which contains information about the negative consequences of overconsumption of alcohol, and that the person can read the leaflet or not ($T$). In general, if the leaflet is read it reduces the amount of alcohol consumed. Assume further that we can intervene on the person to increase her physical activity through two different mechanisms: Either sending her to a bootcamp or motivating her to increase the physical activity on her own. In this scenario, a CSI is introduced such that $T \rightarrow C$ is lost if the person achieves a high physical activity due to the intervention mechanism of sending her to a bootcamp (no alcohol available at the bootcamp, thus reading the leaflet has no effect on alcohol consumption). However, the same is not true for the alternative intervention mechanism (motivating to increase the physical activity does not reduce access to alcohol, thus reading the leaflet does have some effect). 

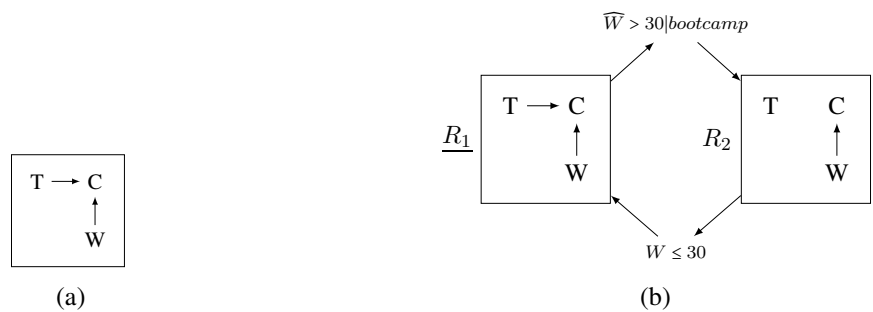
\begin{figure}[t]
	\centering
	\begin{minipage}[t]{0.49\textwidth}
		\begin{subfigure}[t]{\textwidth}
		\centering
		\resizebox{!}{1.5cm} {

\begin{tikzpicture}
	\begin{scope}
		\node[var] (T) {T};
		\node[var] (C)[right=of T] {C};
		\node[var] (W)[below= of C] {W};
		\path (T) edge [connect] (C);
		\path (W) edge [connect] (C);
		\node[box, fit= (T) (C) (W)] (O1) {};
	\end{scope}
\end{tikzpicture}
		}
		\caption{}
		\label{fig:examples-mechanism-truth}
		\end{subfigure}
	\end{minipage}
	\hfill
	\begin{minipage}[t]{0.49\textwidth}
		\begin{subfigure}[t]{\textwidth}
		\centering
		\resizebox{!}{3.5cm} {

\resizebox{1.0\linewidth}{!} {
\centering
		
\begin{tikzpicture}

	\begin{scope}
		\node[var] (T) {T};
		\node[var] (C)[right=of T] {C};
		\node[var] (W)[below= of C] {W};
		\path (T) edge [connect] (C);
		\path (W) edge [connect] (C);
		\node[box, fit= (T) (C) (W), label=left:$\underline{R_1}$] (O1) {};
		
		\node[gate,above=of O1,xshift=2.0cm,yshift=-0.5cm] (G1) {\scriptsize $\wh{W} > 30 | bootcamp $};
		\path (O1) edge [connect] (G1);
	\end{scope}
	
	\begin{scope}[xshift=4.0cm,yshift=0.0cm]
		\node[var] (T) {T};
		\node[var] (C)[right=of T] {C};
		\node[var] (W)[below= of C] {W};
		\path (W) edge [connect] (C);
		\node[box, fit= (T) (C) (W), label=left:$R_2$] (I1) {};
		
		\node[gate,below=of I1,xshift=-2.0cm,yshift=0.5cm] (G3) {\scriptsize $W \leq 30$};
		\path (I1) edge [connect] (G3);
	\end{scope}
	
	\path (G1) edge [connect] (I1);
	\path (G3) edge [connect] (O1);

\end{tikzpicture}

}
		}
		\caption{}
		\label{fig:examples-mechanism-gated}
		\end{subfigure}
	\end{minipage}
	\caption{In (a), the single causal model does not encode the mechanism dependent context, however in (b) the mechanism used to set the context is part of the context itself, thus different outcomes are achieved depending on the mechanism used.}
\end{figure}

As before, using a single causal model (\figurename~\ref{fig:examples-mechanism-truth}) to represent this scenario results in a less informative model than what we can express with a gated model. In \figurename~\ref{fig:examples-mechanism-gated} we have depicted a gated model which can represent such a scenario, where we have used $\wh{W} > 30 | bootcamp$ to mean that the bootcamp mechanism was used.

\subsubsection{Using Gated Models to Identify Causal Effects}

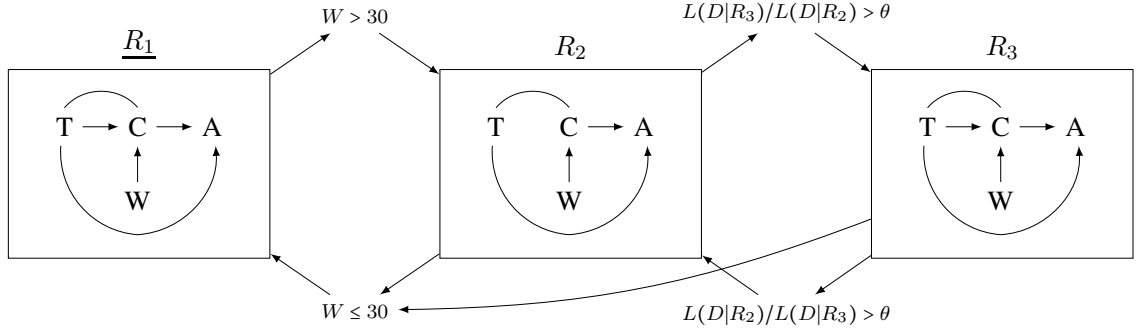
\begin{figure}[t]
	\centering
	\begin{minipage}[t]{1.0\textwidth}
		\centering
		\resizebox{\linewidth}{!} {

\begin{tikzpicture}

	\begin{scope}
		\node[var] (T) {T};
		\node[var, right=of T] (C) {C};
		\node[var, right=of C] (A) {A};
		\node[var, below=of C] (W) {W};
	
		\path (T.north) edge [-, bend left= 50] (C.north);
		\path (T) edge [connect] (C);
		\path (C) edge [connect] (A);
		\path (W) edge [connect] (C);
	
		\path [draw, -latex] (T) to[bend right= 45]  ([yshift=-0.2cm]W.south) to[bend right = 45] (A);

		\node[box, inner sep=15pt, fit= (T) (C) (W) (A), label=above:\underline{$R_1$}] (O1) {};
	\end{scope}
	
	\begin{scope}[xshift=6.0cm,yshift=0.0cm]
		\node[var] (T) {T};
		\node[var, right=of T] (C) {C};
		\node[var, right=of C] (A) {A};
		\node[var, below=of C] (W) {W};
	
		\path (T.north) edge [-, bend left= 50] (C.north);

		\path (C) edge [connect] (A);
		\path (W) edge [connect] (C);
	
		\path [draw, -latex] (T) to[bend right= 45]  ([yshift=-0.2cm]W.south) to[bend right = 45] (A);

		\node[box,  inner sep=15pt, fit= (T) (C) (W) (A), label=above:$R_2$] (I1) {};	
	\end{scope}
	
	\begin{scope}[xshift=12.0cm,yshift=0.0cm]
		\node[var] (T) {T};
		\node[var, right=of T] (C) {C};
		\node[var, right=of C] (A) {A};
		\node[var, below=of C] (W) {W};
	
		\path (T.north) edge [-, bend left= 50] (C.north);

		\path (T) edge [connect] (C);
		\path (C) edge [connect] (A);
		\path (W) edge [connect] (C);
	
		\path [draw, -latex] (T) to[bend right= 45]  ([yshift=-0.2cm]W.south) to[bend right = 45] (A);

		\node[box,  inner sep=15pt, fit= (T) (C) (W) (A), label=above:$R_3$] (O2) {};	
	\end{scope}

	\node[gate,above = of O1,xshift=3.0cm,yshift=-0.5cm] (G1) {\scriptsize $W  > 30$};
	\path (O1) edge [connect] (G1);
	\path (G1) edge [connect] (I1);

	\node[gate,below = of O1,xshift=3.0cm,yshift=0.5cm] (G4) {\scriptsize $W  \leq 30$};
	\path (I1) edge [connect] (G4);
	\path (G4) edge [connect] (O1);
	
	\node[gate,above=of I1,xshift=3.0cm,yshift=-0.5cm] (G2) {\scriptsize $L(D | R_3) / L(D | R_2) > \theta$};
	\path (I1) edge [connect] (G2);
	\path (G2) edge [connect] (O2);	

	\node[gate,below=of I1,xshift=3.0cm,yshift=0.5cm] (G3) {\scriptsize $L(D | R_2) / L(D | R_3) > \theta$};
	\path (O2) edge [connect] (G3);
	\path (G3) edge [connect] (I1);
	\path (O2) edge [connect, bend left=10] (G4);

\end{tikzpicture}
		}
		\caption{A gated model using aADMGs. The causal effect $p(A, C | \wh{T})$ cannot be identified in $R_1$. However, exploiting certain CSIs the effect is identifiable in the context $R_2$.}
		\label{fig:examples-alcohol-activity-admg-gated}
	\end{minipage}
\end{figure}

In this section we shall expand upon the alcohol and physical activity example given in Section~\ref{sec:examples-mechanism-dependent}. This time we shall focus on the identification of causal effects within the contexts, and show how gated models allow us to identify causal effects that cannot be identified using a single causal model even though, unlike in the scenarios described above, none of the gates refers to an intervention.

Consider the causal scenario depicted in \figurename~\ref{fig:examples-alcohol-activity-admg-gated}, which is now represented using a gated model with aADMGs. Here, $C$ represent the monthly consumption of alcoholic beverages for an individual, $W$ represents the individual's monthly level of physical activity measured in number of hours, $A$ represents the number of alcohol related accidents that the individual has been subjected to, and $T$ represents whether or not the individual has read a leaflet containing information about the negative effects of overconsumption of alcoholic beverages, and information about how to avoid accidents when consuming alcohol. From the initial context $R_1$, we can tell that there is confounding between $T$ and $C$, which implies that we cannot identify $p(A, C | \wh{T})$ by Theorem \ref{the:children2}.

We further assume that individuals who have a high level of physical activity do not change the number of alcoholic beverages they consume if they read the leaflet. By exploiting this CSI in $R_2$, $p(A, C | \wh{T})$ can in this context be identified by Theorem \ref{the:children}. Thus the investigator can collect observational data in this context, and use the {\it do}-calculus presented in Section \ref{sec:identification} to calculate $p(A, C | \wh{T})$. If the individual no longer maintains a high level of physical activity, then the gated model expresses that such effect no longer can be calculated from observational data. However, the gated model also expresses that the effect of high physical activity is not stable, and that under certain circumstances (due to unobserved variables, hence the use of threshold gates), the edge $T \rightarrow C$ might be recovered even when $W > 30$. The investigator must therefore be careful, as it may be the case that $R_3$ is more appropriate than $R_2$, thus the effect would again not be identifiable. Therefore, as the investigator collects observations in the context of $W > 30$ it is necessary to assess which of $R_2$ and $R_3$ is most appropriate.

\section{Conclusions}\label{sec:conclusions}

In this paper, we have introduced ADMGs as a combination of oADMGs and aADMGs. We have defined the global Markov property for ADMGs by introducing two equivalent separation criteria. We plan to define local and pairwise Markov properties for ADMGs, and study Markov equivalence between ADMGs.

We have also shown the suitability of ADMGs to represent causal models with additive error terms. We have presented sufficient graphical criteria for the identification of arbitrary causal effects from aADMGs. Some criteria are based on a calculus, while others are based on a decomposition. We plan to study the equivalence of these two sets of criteria, and extend them to ADMGs. We have also presented a necessary and sufficient graphical criterion for the identification of the causal effect of a single variable on the rest of the variables in aADMGs. In the future, we would like to extend this criterion to arbitrary causal effects, as it has been done for oADMGs \citep{HuangandValtorta2006,ShpitserandPearl2006}, and generalize it to ADMGs. 

It is worth mentioning that several of the graphical criteria presented in this paper aim to identify causal effects of the form $p(Y | \wh{X}, W)$, also called conditional causal effects. Such effects are important for the evaluation of conditional and stochastic plans, i.e. plans that intervene on $X$ to set it to a value that depends on the observed value of $W$ either deterministically or stochastically \citep[Chapter 4]{Pearl2009}.

Finally, we have presented gated models for causal effect identification, a new family of graphical models that exploits CSIs to identify additional causal effects. Gated models can be combined with any existing family of graphical models, e.g. oADMGs, aADMGs, ADMGs, etc. We have illustrated the benefits of gated models with examples of certain causal phenomena that many practitioners may encounter when dealing with real-world systems, for instance uncertainty in the outcome of interventions and unstable effects. A line of research that we are interested in pursuing is the extension of the exact learning algorithm presented in Section \ref{sec:learning} to learn gated models from observations and interventions.

\section*{Acknowledgments}

We thank the Reviewers for their comments, which helped us to improve our manuscript.

\bibliography{pena16aIJARv2}
\end{document}